%% file: submission.tex
\pdfoutput=1
\documentclass{article}

\input{settings/preamble.tex}

\begin{document}
\input{settings/author-info.tex}
\input{sections/abstract.tex}

\input{sections/introduction.tex}
\input{sections/background/main.tex}

\input{sections/methods/main.tex}
\input{sections/validations/main.tex}

\input{sections/experiments/main.tex}

\input{sections/conclusions.tex}

\clearpage
\input{sections/acknowledgements.tex}

\bibliographystyle{bib-style}
\bibliography{ref}

\input{settings/appendix-information.tex}

\end{document}

%% file: settings/preamble.tex
\input{settings/packages.tex}
\input{settings/paper-status.tex}

\input{settings/user-commands.tex}
\input{settings/layout.tex}

%% file: settings/packages.tex
\usepackage{ijcai22}

\usepackage{times}
\usepackage{soul}
\usepackage[small]{caption}
\usepackage{graphicx}
\usepackage{booktabs}
\usepackage{array}

\usepackage[hidelinks]{hyperref}  
\usepackage[utf8]{inputenc}
\usepackage{url}

\usepackage{amsmath}
\usepackage{amssymb}
\usepackage{bbm}
\usepackage{bm}
\usepackage{cancel}
\usepackage{mathtools}
\usepackage{xfrac}

\usepackage{ascmac}
\usepackage{algorithm}
\usepackage{algpseudocode}

\usepackage{fancyhdr}
\usepackage{tcolorbox}
\usepackage{tikz}

\usepackage{booktabs}
\usepackage{comment}
\usepackage{lastpage}
\usepackage{listings}
\usepackage{multibib}
\usepackage{multirow}
\usepackage[super]{nth}
\usepackage[caption=false]{subfig}

\tcbuselibrary{breakable, skins, theorems}

\allowdisplaybreaks

\newtheorem{theorem}{Theorem}

\newtheorem{proposition}{Proposition}

\newtheorem{lemma}{Lemma}

\newtheorem{definition}{Definition}

\newtheorem{proof}{Proof}

\algtext*{EndFor}
\algtext*{EndWhile}
\algtext*{EndIf}
\algtext*{EndProcedure}
\algtext*{EndFunction}
\algnewcommand{\LineComment}[1]{\State \(\triangleright\) #1}

\newcites{appx}{References}

%% file: settings/paper-status.tex
\newif\ifunderreview
\newif\ifsubmission
\newif\ifappendix

\underreviewfalse

\submissiontrue

\appendixtrue

\ifsubmission
\newcommand{\todo}[1]{}
\newcommand{\replace}[2]{}
\newcommand{\sw}[1]{}
\newcommand{\fh}[1]{}
\newcommand{\archit}[1]{}

\else
\newcommand{\todo}[1]{\textbf{\textcolor{red}{[TODO: #1]}}}

\newcommand{\replace}[2]{\textbf{\textcolor{red}{[del: \cancel{#1}]}}\textbf{\textcolor{blue}{[new: #2]}}}

\newcommand{\sw}[1]{\textbf{\textcolor{cyan}{[SW: #1]}}}

\newcommand{\fh}[1]{\textbf{\textcolor{brown}{[FH: #1]}}}

\newcommand{\archit}[1]{\textbf{\textcolor{olive}{[AB: #1]}}}

\usepackage[inline]{showlabels}

\fi

%% file: settings/user-commands.tex
\makeatletter
\newcommand{\customlabel}[2]{%
   \protected@write \@auxout {}{\string \newlabel {#1}{{#2}{\thepage}{#2}{#1}{}} }%
   \hypertarget{#1}{}
}
\makeatother

\newcommand{\xv}{\boldsymbol{x}}

\newcommand{\xopt}{\xv^{\text{opt}}}

\newcommand{\B}{\mathcal{B}}
\newcommand{\D}{\mathcal{D}}
\newcommand{\Pcal}{\mathcal{P}}
\newcommand{\Scal}{\mathcal{S}}
\newcommand{\X}{\mathcal{X}}

\newcommand{\Bx}{\B_{\X}}
\newcommand{\Xg}{\X^{\gamma}}
\newcommand{\Xgp}{\X^{\gamma^\prime}}

\newcommand{\ceil}[1]{\lceil #1 \rceil}

\newcommand{\pe}[2]{D_{\mathrm{PE}}(#1 \| #2)}

\newcommand{\indic}[1]{\mathbb{I}[#1]}

\newcommand{\citewithname}[1]{\citeauthor{#1}~\shortcite{#1}}

\newcommand{\wm}{\texttt{Width multiplier}}
\newcommand{\ta}{\texttt{TrivialAugment}}

%% file: settings/layout.tex
\fancypagestyle{customheader}{
    \fancyhf{} 
    \fancyhead[C]{\raisebox{6ex}[0pt]{\small Proceedings of the Thirty-Second International Joint Conference on Artificial Intelligence (IJCAI-23)}}
}

%% file: settings/author-info.tex
\title{
    PED-ANOVA: Efficiently Quantifying Hyperparameter Importance \\ in Arbitrary Subspaces \\
}

\author{
  Shuhei Watanabe
  \and
  Archit Bansal
  \And
  Frank Hutter \\
  \affiliations
  Department of Computer Science, University of Freiburg, Germany \\
  \emails
  \{watanabs,bansala,fh\}@cs.uni-freiburg.de \\
}

\maketitle

%% file: sections/abstract.tex
\begin{abstract}
  The recent rise in popularity of Hyperparameter Optimization (HPO) for deep learning has highlighted the role that good hyperparameter (HP) space design can play in training strong models.
  In turn, designing a good HP space is critically dependent on
  understanding the role of different HPs.
  This motivates research on HP Importance (HPI), e.g., with the popular method of functional ANOVA {(f-ANOVA)}.
  However, the original f-ANOVA formulation is inapplicable to the subspaces most relevant to algorithm designers, such as those defined by top performance.
  To overcome this issue, we derive a novel formulation of {f-ANOVA} for arbitrary subspaces and propose an algorithm that uses Pearson divergence (PED) to enable a closed-form calculation of HPI.
  We demonstrate that this new algorithm, dubbed \emph{PED-ANOVA}, is able to successfully identify important HPs in different subspaces while also being extremely computationally efficient.
\end{abstract}

%% file: sections/introduction.tex
\section{Introduction}
Following on the heels of widespread adoption of deep learning models in various industries and areas of research, Hyperparameter Optimization 
(HPO)~\cite{bergstra2012random,snoek2012practical,bergstra2011algorithms,lindauer2021smac3,watanabe2023tpe}
for deep learning has gained increasing prominence as the path forward for making deep learning more accessible and robust. In particular, recent research has highlighted the role that good hyperparameter (HP) space design can play in training strong models~\cite{chen2018bayesian,melis2018on,henderson-aaai18a}.
In practice, while a large search space is necessary to find high-performance models~\cite{zimmer2021auto},
a reduced search space that retains only important HPs is essential for efficiently finding them~\cite{perrone2019learning}.
Therefore, it is crucial to understand the role that different HPs play in a search space.

\begin{figure}
  \centering
  \includegraphics[width=0.48\textwidth]{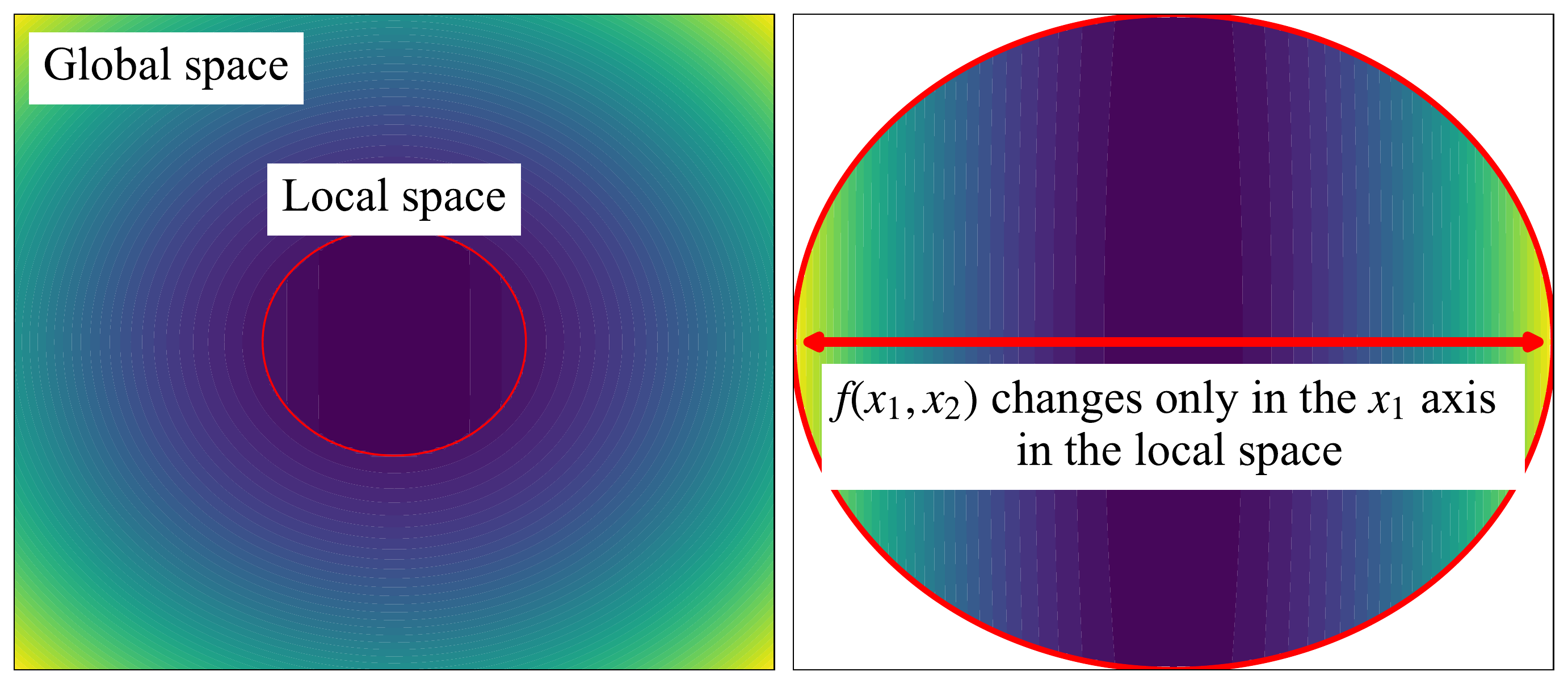}
  \caption{
    An example where the trend of HPI changes in the global and local space (top-$10\%$).
    The horizontal axis is the $x_1$-axis, the vertical axis is the $x_2$-axis,
    and $f(x_1, x_2)$ is the objective function to analyze (darker is better).
    \textbf{Left}: the normalized contour plot of $f(x_1, x_2)$ in global space;
    both $x_1$ and $x_2$ appear to variate $f(x_1, x_2)$ equally.
    The red circle is the promising domain we would like to explore.
    \textbf{Right}: the normalized contour plot of $f(x_1, x_2)$ in the local space
    (the red circle in global space);
    $f(x_1, x_2)$ variates only by $x_1$, and thus $x_1$ is more important.
    Such a trend cannot be captured by existing methods.
  }
  \vspace{-4mm}
  \label{main:intro:fig:when-local-is-important}
\end{figure}

\begin{figure*}[t]
  \begin{center}
    \subfloat[\sloppy The true $\gamma$-set of $f(x_1, x_2) = x_1^2 + x_2^2$]{
      \includegraphics[width=0.35\textwidth]{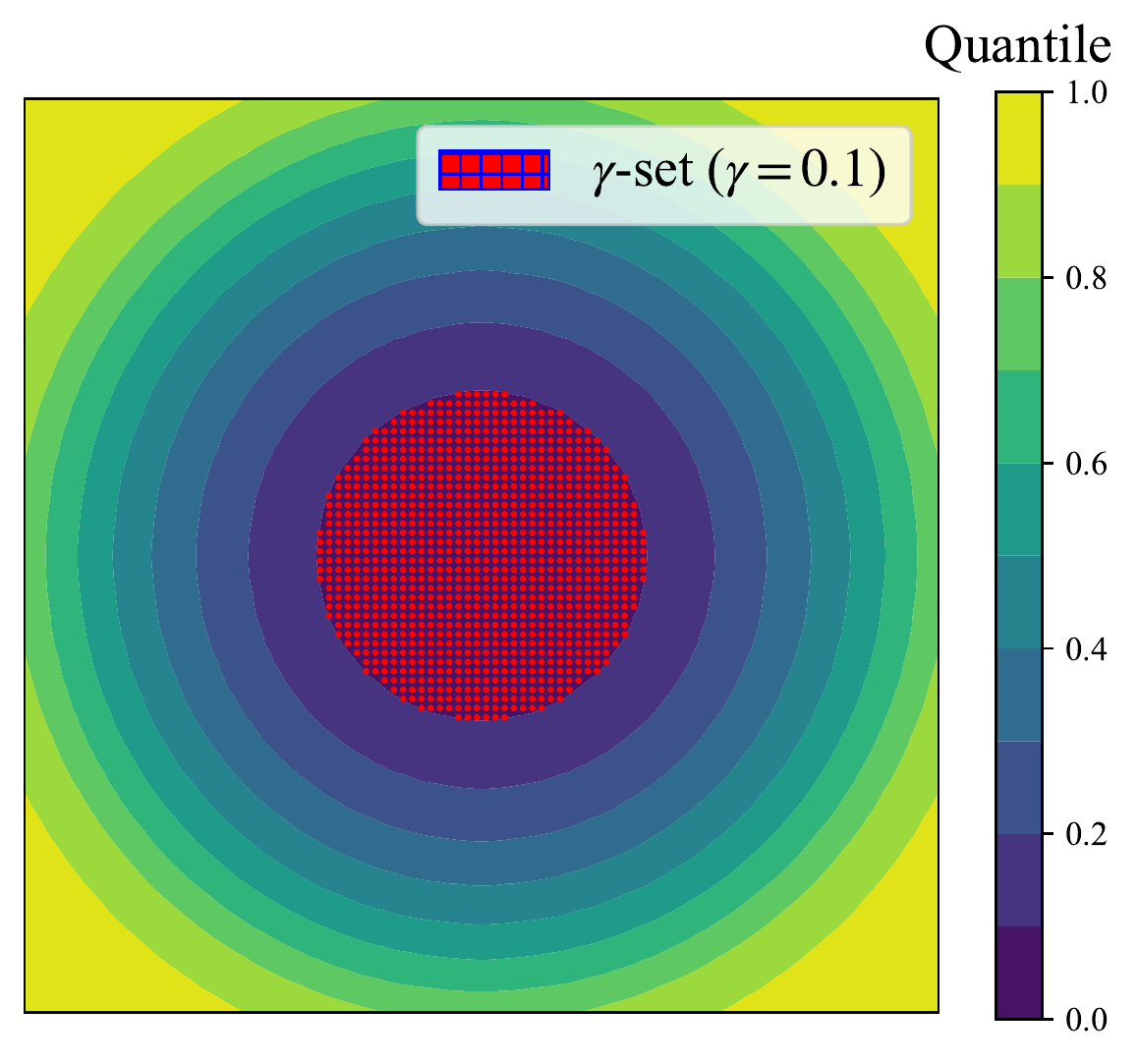}
    }
    \subfloat[\sloppy An empirical $\gamma$-set PDF using KDE of $f(x) = \sin x$]{
      \includegraphics[width=0.60\textwidth]{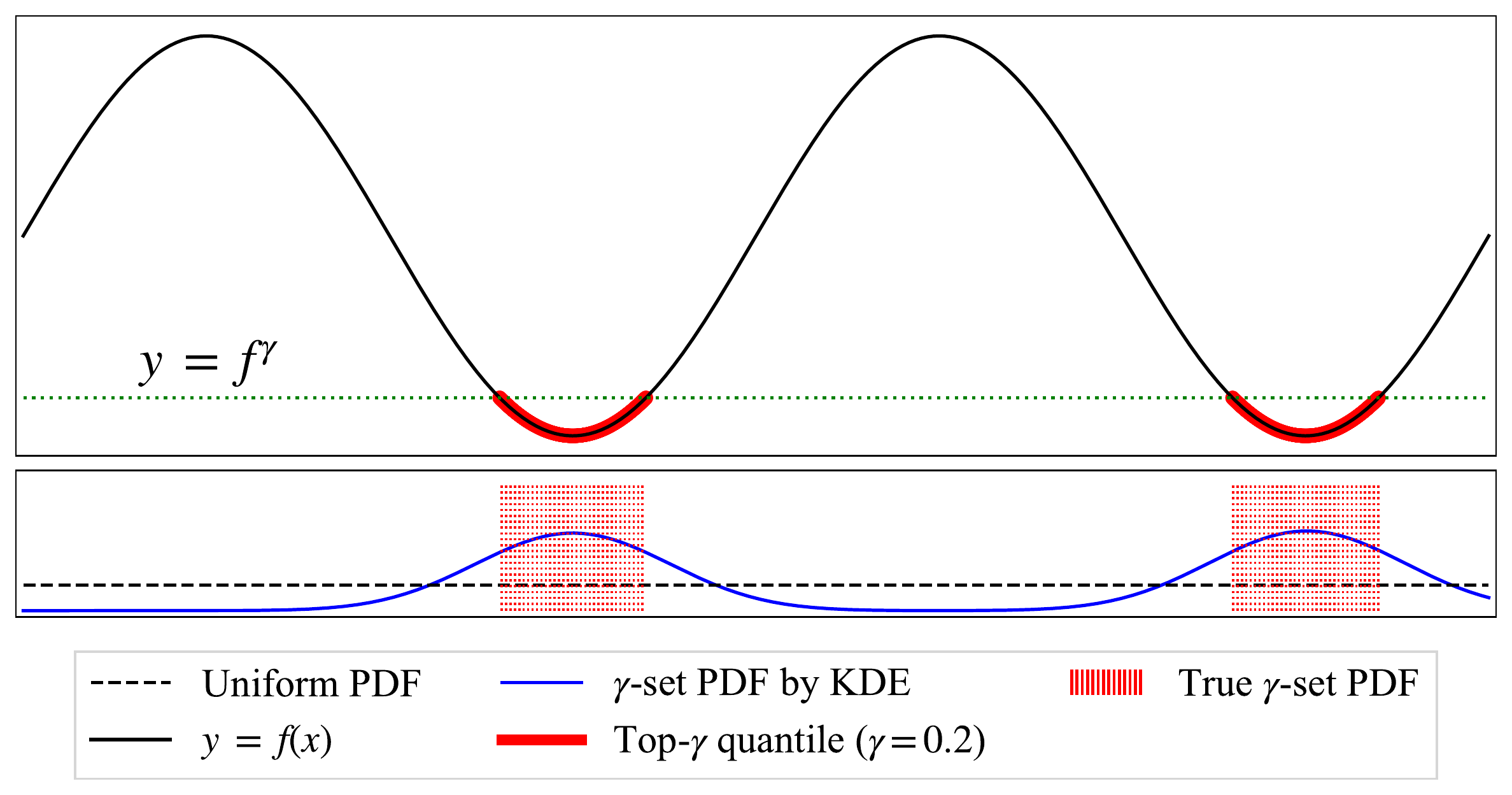}
    }
    \vspace{-2mm}
    \caption{
      Conceptual visualizations of the $\gamma$-set
      and the $\gamma$-set PDF.
      \textbf{Left}:
      the true $\gamma$-set in a 2D example.
      Darker is better in the figure and 
      we colored the top-$10\%$ in red, which is the $\gamma$-set $\Xg$ ($\gamma=0.1$)
      in this example.
      \textbf{Right}:
      the true $\gamma$-set and the $\gamma$-set PDF in a 1D example.
      The green dotted line shows the $\gamma$-quantile value $f^\gamma$,
      which achieves the top-$20\%$ in this example,
      and the red solid lines
      are the $\gamma$-set $\Xg$ ($\gamma=0.2$).
      The red dotted spaces below show the
      true $\gamma$-set PDF $p(\xv|\Xg)$,
      but 
      since we do not have the analytical form
      in practice,
      this PDF is estimated by KDE (the blue solid line).
      \label{main:background:fig:gamma-set-conceptual-viz}
    }
    \vspace{-4mm}
  \end{center}
\end{figure*}

This is the driving force behind previous research into the quantification of HP Importance (HPI)~\cite{hutter2014efficient,biedenkapp2017efficient},
which still remains a largely understudied section of HPO research. 
Several HPO frameworks~\cite{biedenkapp2018cave,akiba2019optuna,sass2022deepcave} have previously utilized functional ANOVA (f-ANOVA)~\cite{hooker2007generalized,hutter2014efficient} to provide a better interpretation of the role of different HPs,
but the original f-ANOVA~formulation is not very practical for the interpretation of specific subspaces of a search space.
Such subspaces are often of particular interest to algorithm developers due to various properties, for example, the ``local space'' visualized in Figure~\ref{main:intro:fig:when-local-is-important} could represent a region of high performance. 
Nevertheless, prior works~\cite{hutter2014efficient,biedenkapp2018cave} have attempted to overcome this and quantify HPI in specific subspaces using f-ANOVA.
However, since their formulation did not constrain the calculations to subspaces of high interest,
we argue that the results are biased towards unimportant subspaces.
At the same time, obtaining an unbiased quantification of HPI in specific subspaces of interest,
which we refer to as \emph{local} spaces in contrast to the full \emph{global} space,
is mathematically non-trivial.

To overcome this issue, we first formally define local HPI as HPI in a local space and
we derive a novel formulation of f-ANOVA to compute local HPI for arbitrary local spaces.
Still, our formulation would require Monte-Carlo sampling in general and it is computationally intractable.
Therefore, we show that local HPI is tractable without a Monte-Carlo sampling under some constraints
and propose an algorithm that uses Pearson divergence (PED, \cite{pearson1900x}) to enable a closed-form computation of HPI.
In a series of experiments,
we first verify that our algorithm correctly provides global and local HPI in a toy function.
Then we demonstrate that our algorithm takes only less than a second for $10^5$ data points
while the prior f-ANOVA~\cite{hutter2014efficient} would take more than a week.

To provide a solid picture of how to use our method,
we perform analysis on JAHS-Bench-201~\cite{bansal2022jahs},
which has one of the largest search spaces among HPO benchmarks.
In the analysis, we find that
it is suboptimal to design a search space relying only on global HPI
because we potentially miss important HPs in a local space if the global HPI of these HPs are dominated by the most important HP.
We demonstrate that local HPI plays a crucial role to avoid this issue.
Furthermore, our method has several other possible applications such as
(1)~post-hoc analysis for HPO,
(2)~adaptive (e.g., meta-learned) search space reductions for faster HPO,
and
(3)~exploratory data analysis.
We discuss these in
more detail in Appendix~\ref{appx:use-cases:section},
along with the advantages and limitations of our method.

In summary, the contributions of this paper are to:
\begin{enumerate}
  \item reformulate local HPI mathematically and derive the general formula of local HPI,
  \item provide a closed-form calculation for local HPI using PED that handles
  even $10^8$ data points in a minute, and
  \item benchmark performance compared with the original f-ANOVA.
\end{enumerate}
To facilitate reproducibility,
our implementation is available at
\ifunderreview
\url{https://anonymous.4open.science/r/local-anova-FBF8/}.
\else
\url{https://github.com/nabenabe0928/local-anova/}.
\fi

%% file: sections/background/main.tex
\section{Background \& Related Work}

\input{sections/background/preliminaries.tex}

\input{sections/background/anova.tex}

\begin{figure}[t]
  \begin{center}
    \subfloat[\sloppy The contour plot of $f(x_1,x_2)=x_1^2 + x_2^2 / 10$ and the marginal means for each dimension.]{
      \includegraphics[width=0.2\textwidth]{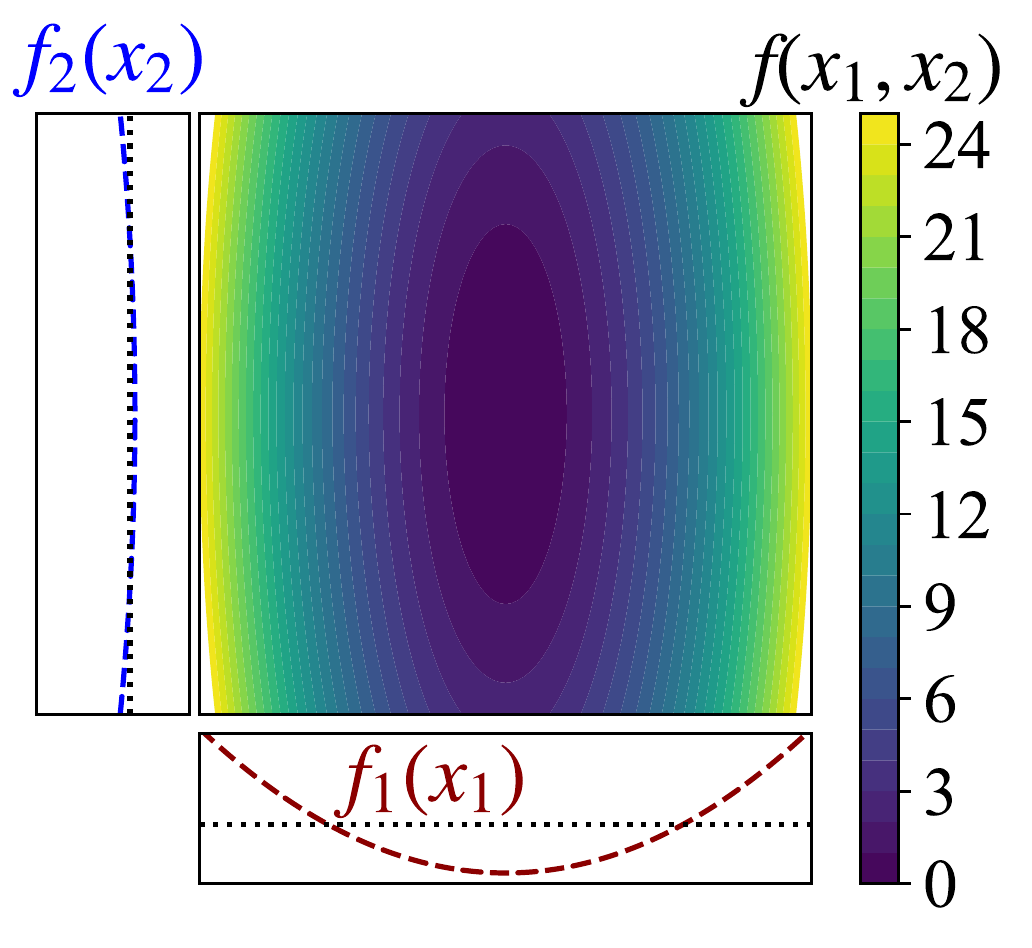}
    }
    \hspace{3mm}
    \subfloat[\sloppy The conceptual visualization of global HPI on $f(x_1, x_2)=x_1^2+x_2^2 / 10$.]{
      \includegraphics[width=0.2\textwidth]{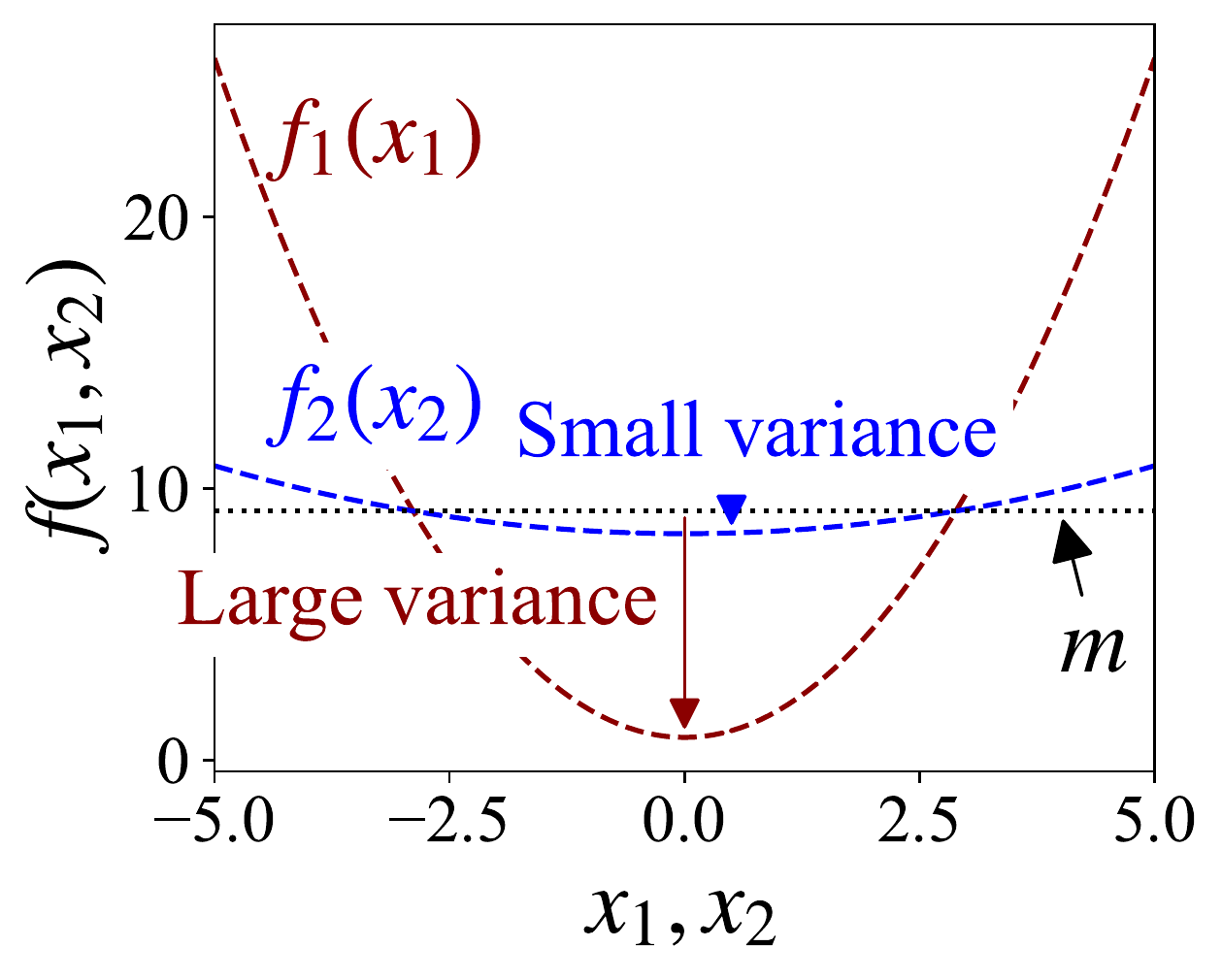}
    }
    \vspace{-1mm}
    \caption{
      The conceptual visualization of global HPI on $f(x_1,x_2)=x_1^2+x_2^2/10$.
      \textbf{Left}: the landscape of $f(x_1, x_2)$ (darker is better)
      and the marginal means $f_1(x_1)$ and $f_2(x_2)$.
      The color does not change a lot in the vertical direction (the $x_2$-axis)
      while it does in the horizontal direction (the $x_1$-axis).
      \textbf{Right}:
      the marginal means and the mean value in the same plane.
      As the marginal mean $f_1$ (the red dashed line)
      has a large variance, $x_1$ is more important.
      On the other hand,
      as the marginal mean $f_2$ (the blue dashed line)
      has a small variance, $x_2$ is less important.
      \label{main:background:fig:anova-intuition}
    }
    \vspace{-4mm}
  \end{center}
\end{figure}

\input{sections/background/local-anova.tex}

%% file: sections/background/preliminaries.tex
\subsection{Preliminaries}
\label{main:background:sec:preliminaries}
Throughout this paper, we use the following terms:
\begin{enumerate}
  \item \textbf{$\gamma$-quantile value $f^\gamma$}:
  The function value $f^\gamma \in \mathbb{R}$
  that achieves the top-$\gamma$ quantile with respect to the objective function $f: \X \rightarrow \mathbb{R}$ to analyze in the global space
  $\X$,
  \item \textbf{$\gamma$-set $\Xg$}:
  A set of configurations $\Xg$ that achieves the top-$\gamma$ quantile
  in the global space $\X$, and
  \item \textbf{Marginal $\gamma$-set PDF $p_d(x_d|\Xg)$}:
  The marginal PDF of the $\gamma$-set PDF $p(\xv|\Xg)$:
  \begin{equation}
  \begin{aligned}
    p_d(x_d|\Xg) \coloneqq
    \int_{\xv_{-d} \in \X_{-d}}
    p(\xv|\Xg) d\xv_{-d}.
  \end{aligned}
  \end{equation}
\end{enumerate}
We provide the formal definitions in Appendix~\ref{appx:theoretical-details:section}
and the conceptual visualizations in Figure~\ref{main:background:fig:gamma-set-conceptual-viz}.
Note that $\X \coloneqq \X_1 \times \dots \times \X_D$ is the search space,
$\X_d \subseteq \mathbb{R}$ for $d \in [D] \coloneqq \{1,\dots,D\}$ is
the domain of the $d$-th HP,
$\xv_{s} \sim \X_{s} \subseteq \mathbb{R}^{|s|}$ denotes $\xv_{s}$ is sampled from
the uniform distribution on $\X_{s}$ where $s \subseteq [D]$,
and $\X_{-d} \subseteq \mathbb{R}^{D - 1}$
is $\X$ without the $d$-th dimension.
Furthermore, we consistently denote the PDF of the uniform distribution
as \emph{uniform PDF}
and follow the assumptions stated in Appendix~\ref{appx:proofs:section:assumptions}.

%% file: sections/background/anova.tex
\subsection{f-ANOVA}
In this section, we describe f-ANOVA for one-dimensional effects
and refer to more details about the general version in Appendix~\ref{appx:background:section:general-anova}.
Suppose we would like to quantify HPI of
a function $f(\xv)$ defined on $\X$, then global HPI~\cite{hooker2007generalized} requires
(see Figure~\ref{main:background:fig:anova-intuition} for the intuition):
\begin{enumerate}
  \item \textbf{Global mean}:
  \begin{equation}
  \begin{aligned}
    m \coloneqq \mathbb{E}_{\xv \sim \X}[f(\xv)], \\
  \end{aligned}
  \label{main:background:eq:global-mean}
  \end{equation}
  \item \textbf{Marginal mean}:
  \begin{equation}
  \begin{aligned}
    f_d(x_d) \coloneqq \mathbb{E}_{\xv_{-d} \sim \X_{-d}}[f(\xv | x_d)], \\
  \end{aligned}
  \label{main:background:eq:marginal-mean}
  \end{equation}
  \item \textbf{Marginal variance}:
  \begin{equation}
  \begin{aligned}
    v_d \coloneqq \mathbb{E}_{x_d \sim \X_d}[(f_d(x_d) - m)^2].
  \end{aligned}
  \label{main:background:eq:marginal-var}
  \end{equation}
\end{enumerate}
Note that $f(\xv | x_d)$ implies that we fix the $d$-th HP
of $\xv$ to $x_d$.
When we denote the global variance as $v_0$,
the ratio $v_d / v_0$ is the \emph{global HPI} of the $d$-th HP
and in essence, the magnitude of the marginal variance represents
the relative importance.

%% file: sections/background/local-anova.tex
\subsection{Local f-ANOVA in Prior Works}
\label{main:background:section:local-anova}
To the best of our knowledge,
there are two papers that mention \emph{local HPI}
(and both use f-ANOVA).
\citewithname{hutter2014efficient} mentioned
local HPI can be quantified by taking:
\begin{equation}
  \begin{aligned}
    g(\xv) \coloneqq \min(f(\xv), f^\gamma);
  \end{aligned}
  \label{main:background:eq:frank-anova-transformation}
\end{equation}
however, since this measure is biased depending on the global space design
as discussed in Appendix~\ref{appx:theoretical-details:section:analysis-clipped},
we need to consider the integral only over a local space as
stated in Section~\ref{main:methods:section:local-hpi}.
\citewithname{biedenkapp2018cave} proposed
the following HPI measure:
\begin{equation}
  \begin{aligned}
    m_d & \coloneqq \mathbb{E}_{x_d \sim \X_d}[ f(\xv | \xopt_{-d})],          \\
    v_d & \coloneqq \mathbb{E}_{x_d \sim X_d}[ (f(\xv | \xopt_{-d}) - m_d)^2], \\
  \end{aligned}
  \label{main:background:eq:local-hpi-biedenkapp}
\end{equation}
where $\xopt \in \mathbb{R}^D$
is the optimized setting and $\xopt_{-d} \in \mathbb{R}^{D - 1}$
is $\xopt$ without the $d$-th dimension.
The authors mention that this measure is a local HPI measure;
however, this measure is also not a local HPI measure in our definition
and we show that this is the case using a toy example in Appendix~\ref{appx:theoretical-details:section:analysis-on-func}.

%% file: sections/methods/main.tex
\section{Local f-ANOVA Using Pearson Divergence}
\label{main:methods:section:local-anova-detail}
In this section, we first provide the definition
of local HPI and describe how to define
a local space.
For simplicity, we name this local space definition as
\emph{Lebesgue split}~\footnote{
  The name comes from the fact that
  we define a local space by a function value as
  in the Lebesgue integral in contrast to
  the definition of a local space by bounds for each dimension,
  which we name \emph{Riemann split}.
}.
Then we introduce fast algorithm using
PED between two KDEs to compute
local HPI and benchmark the speed of
the algorithm compared to
the f-ANOVA implementation
based on random forests~\cite{hutter2014efficient}.
Notice that since higher orders of HPI
require exponential amounts of computations
and usually lack interpretability,
our discussion does not focus on higher orders;
however, we derive the formula for higher orders
and show them in
Eqs.~(\ref{appx:proofs:eq:general-global-hpi}),
(\ref{appendix:proofs:eq:general-local-importance})
in Appendix~\ref{appx:proofs:section}.
The theoretical details for this section are available
in Appendix~\ref{appx:theoretical-details:section}.

\input{sections/methods/local-importance.tex}

\input{sections/methods/pseudo-for-anova.tex}

\input{sections/methods/quick-algo.tex}

%% file: sections/methods/local-importance.tex
\subsection{Local Hyperparameter Importance}
\label{main:methods:section:local-hpi}
In this section, we assume that we have a set of (sorted) observations
$\D \coloneqq \{(\xv_n, f(\xv_n))\}_{n=1}^N$
such that $f(\xv_1) \leq \dots \leq f(\xv_N)$.
Then, the top-$\gamma$-quantile observations
are $\D^\gamma = \{(\xv_n, f(\xv_n))\}_{n=1}^{\ceil{\gamma N}}$
and the $\gamma$-quantile value is $f^\gamma \coloneqq f(\xv_{\ceil{\gamma N}})$.

\subsubsection{Local Space Defined by Lebesgue Split}
To begin with, we formally define local HPI:
\begin{definition}[Local HPI]
  Given a subspace $\X^\star \subseteq \X$,
  local $\mathrm{HPI}$ is
  global $\mathrm{HPI}$ $v_d/v_0$ in the subspace $\X^\star$.
  \label{main:methods:def:local-hpi}
\end{definition}
Recall that $v_d$ is the marginal variance of the $d$-th
HP and $v_0$ is the global variance.
Based on Definition~\ref{main:methods:def:local-hpi},
the prior works on making f-ANOVA local
discussed in Section~\ref{main:background:section:local-anova}
are not local HPI measures;
see Appendix~\ref{appx:theoretical-details:section:difference-from-prior-works} for more details.
As our local HPI obviously depends on the choice of $\X^\star$,
local HPI is a very general concept;
therefore, we focus on the so-called \emph{Lebesgue split}
to specify a local space in this paper.
In the Lebesgue split, we obtain a local space $\X^\star$ as follows:
\begin{enumerate}
  \item Fix a threshold $f^\star$ (we use the $\gamma$-quantile value $f^\gamma$ in this paper instead), and
        \vspace{-1mm}
  \item Obtain the sublevel set
        $\X^\star \coloneqq \{\xv \in \X | f(\xv) \leq f^\star\}$
        based on $f^\star$ ($\X^\star$ becomes the $\gamma$-set $\Xg$ when $f^\star = f^
          \gamma$).
\end{enumerate}
Recall that the definitions of
the $\gamma$-quantile value and
the $\gamma$-set are available in Section~\ref{main:background:sec:preliminaries}.
More intuitively, the red domains in Figure~\ref{main:background:fig:gamma-set-conceptual-viz}
are the local space of each example.
In this paper, we use $f(\xv_{\ceil{\gamma N}})$ as $f^\gamma$.
The advantages of the Lebesgue split are to:
\begin{enumerate}
  \item require only one parameter $f^\star$ while
        the Riemann split, which we split along each dimension by specifying bounds, requires
        at least $2 \times D$ parameters,
        \vspace{-1mm}
  \item be able to focus on the analysis in promising domains where we are interested, and
        \vspace{-1mm}
  \item be able to remove the sampling bias caused by a non-uniform sampler when using the formula of local HPI.
\end{enumerate}
We further discuss the strengths and drawbacks of the Lebesgue split
compared to the Riemann split in Appendix~\ref{appx:theoretical-details:section:splits}.

\begin{figure}[t]
  \begin{center}
    \subfloat[\sloppy  The contour plot of $f(x_1,x_2)$ and the marginal $\gamma^\prime$-set PDFs.]{
      \includegraphics[width=0.2\textwidth]{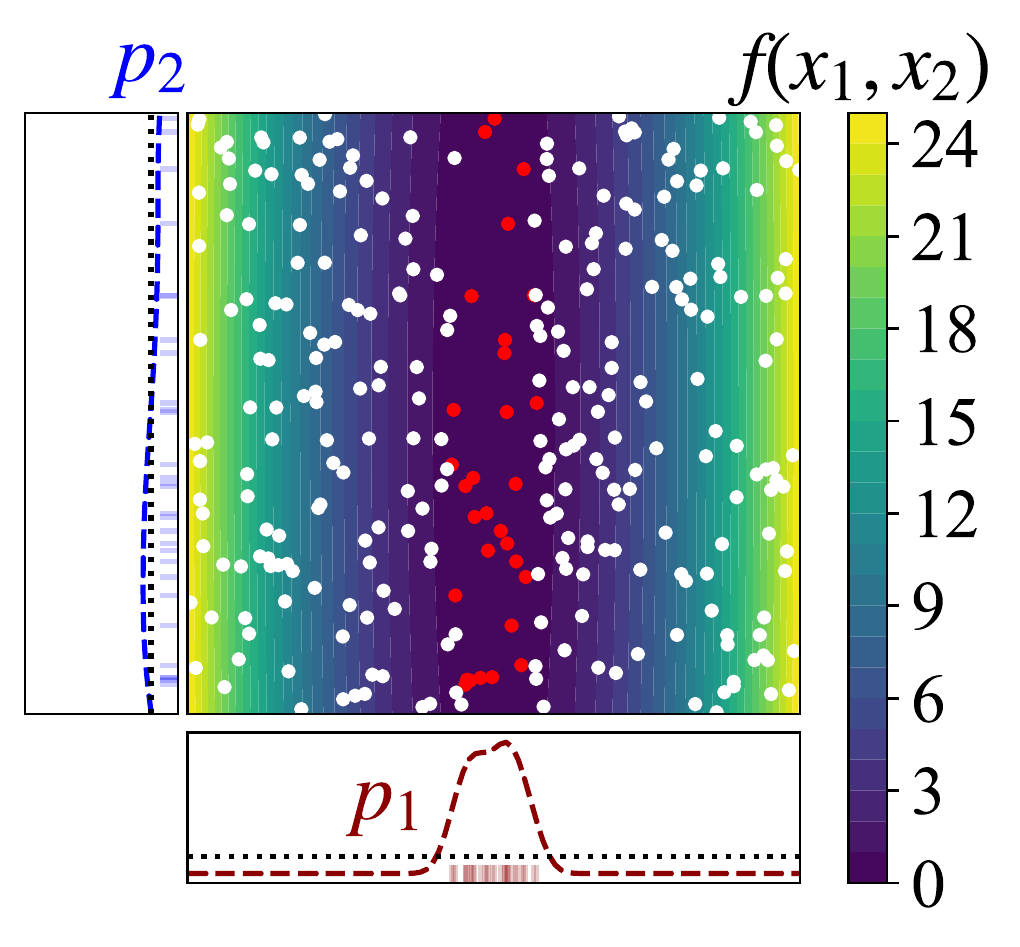}
    }
    \hspace{3mm}
    \subfloat[\sloppy The conceptual visualization of global HPI by PED.]{
      \includegraphics[width=0.2\textwidth]{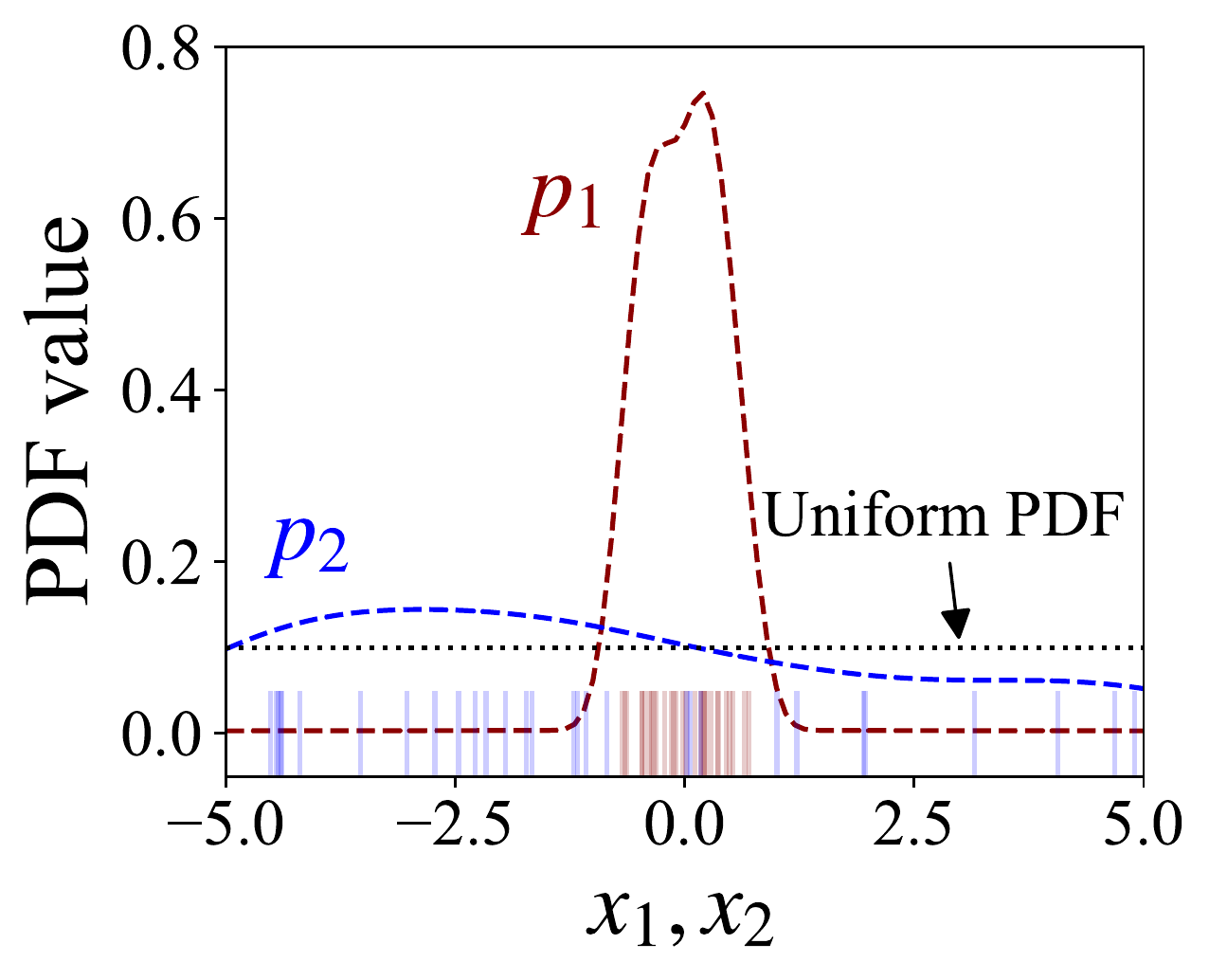}
    }
    \caption{
      The conceptual visualization of global HPI using PED on $f(x_1,x_2)=x_1^2+x_2^2/100$.
      As we consider global HPI in this example, the $\gamma$-set PDF is the uniform PDF.
      \textbf{Left}: the landscape of $f(x_1, x_2)$ (darker is better)
      and the marginal $\gamma^\prime$-set PDFs
      $p_1(x_1|\D^{\gamma^\prime})$ ($=p_1$, the red dashed line) and $p_2(x_2|\D^{\gamma^\prime})$ ($=p_2$, the blue dashed line).
      The dots represent observations (datasets)
      and the red dots achieve the top-$\gamma^\prime$ quantile.
      $p_1, p_2$ are estimated
      by Eq.~(\ref{main:methods:eq:compressed-kde}) using the red dots.
      \textbf{Right}:
      the marginal $\gamma^\prime$-set PDFs in the same plane.
      While $p_1$ (the red dashed line) sharply peaks at the center,
      $p_2$ (blue dashed line) is close to the uniform PDF.
      This implies that the latter is closer to
      the uniform PDF, and thus $x_2$ is less important.
      \label{main:methods:fig:conceptual-ped-anova}
    }
    \vspace{-4mm}
  \end{center}
\end{figure}

\subsubsection{Formula of Local Hyperparameter Importance}
\label{main:methods:section:formula-of-local-anova}
Now we discuss the computation of local HPI.
In Eqs.~(\ref{main:background:eq:global-mean})--(\ref{main:background:eq:marginal-var}),
we take the expectation over the uniform distribution
of the global space $\X$.
In the same vein, it is natural to consider
the expectation over the uniform distribution
of the local space $\Xg$ for local HPI as well.
Although the computation is not obvious, we can compute the expectation of a measurable function $f(\xv)$ over the local space $\Xg$ if we use the following trick:
\begin{equation}
  \begin{aligned}
    \frac{1}{\gamma} \mathbb{E}_{\xv \sim \X}[
    f(\xv)\indic{\xv \in \Xg}
    ],
  \end{aligned}
\end{equation}
where $\gamma = \mathbb{E}_{\xv \sim \X}[\indic{\xv \in \Xg}]$ is a normalization constant.
Recall that $\xv \in \Xg$ and $f(\xv) \leq f^\gamma$
are equivalent.
Similarly, the marginal mean of $f(\xv)$ is computed as follows:
\begin{equation}
  \begin{aligned}
    \frac{1}{V^{\gamma}_d(x_d)} \mathbb{E}_{\xv_{-d} \sim \X_{-d}}[
    f(\xv | x_d)\indic{\xv \in \Xg | x_d}
    ],
  \end{aligned}
\end{equation}
where $V^{\gamma}_d(x_d) \coloneqq \mathbb{E}_{\xv_{-d} \sim \X_{-d}}[
  \indic{\xv \in \Xg | x_d}
  ]$ is a normalization constant.
Then local HPI is generally computed as follows:
\begin{enumerate}
  \item \textbf{Local mean}:
        \begin{equation}
          \begin{aligned}
            m^\gamma \coloneqq \frac{1}{\gamma} \mathbb{E}_{
            \xv \sim \X}[f(\xv)\indic{\xv \in \Xg}], \\
          \end{aligned}
          \label{main:methods:eq:local-mean}
        \end{equation}
  \item \textbf{Local marginal mean}:
        \begin{equation}
          \begin{aligned}
            f_d^\gamma(x_d) \coloneqq \frac{1}{V^{\gamma}_d(x_d)}
            \mathbb{E}_{\xv_{-d} \sim \X_{-d}}[
            f(\xv | x_d)\indic{\xv \in \Xg |x_d}], \\
          \end{aligned}
          \label{main:methods:eq:marginal-local-mean}
        \end{equation}
  \item \textbf{Local marginal variance}:
        \begin{equation}
          \begin{aligned}
            v_d^\gamma \coloneqq
            \mathbb{E}_{x_d \sim V^{\gamma}_d}[
            (f_d^\gamma(x_d) - m^\gamma)^2].
          \end{aligned}
          \label{main:methods:eq:marginal-local-var}
        \end{equation}
\end{enumerate}
Note that $x_d \sim V^{\gamma}_d$ implies that
$x_d$ is sampled from the distribution of
the PDF $V^{\gamma}_d(x_d) / Z$ where $Z \in \mathbb{R}_{+}$ is
the normalization constant.
As the series of computations requires
a Monte-Carlo sampling in a $D - 1$ dimensional space,
the time complexity incurs the curse of dimensionality.
In the next section, we introduce
fast algorithm to compute local HPI
in exchange for the scale ignorance.

%% file: sections/methods/pseudo-for-anova.tex
\begin{algorithm}[tb]
  \caption{Local PED-ANOVA}
  \label{main:methods:alg:anova}
  \begin{algorithmic}[1]
    \Statex{$\D = \{(\xv_n, f(\xv_n))\}_{n=1}^{N}$ (Dataset to analyze), $\gamma, \gamma^\prime$ (User-defined quantiles of top domains)}
    \LineComment{See Appendices~\ref{appx:use-cases:section:posthoc-hpo}, \ref{appx:use-cases:section:space-reduction} for practical usages}
    \State{Sort $\D$ in ascending order by $f$}
    \LineComment{\textcolor{magenta}{$|\D^{\gamma}| \geq 2$ and $|\D^{\gamma^\prime}| \geq 2$ must hold}}
    \State{Pick the top-$\gamma$ and -$\gamma^\prime$ quantile observations $\D^{\gamma}, \D^{\gamma^\prime}$}
    \For{$d = 1, \dots, D$}
    \State{Count occurrences of unique values $c_d^{(n)}$}
    \State{Build KDEs $p_d(\cdot|\D^\gamma), p_d(\cdot|\D^{\gamma^\prime})$ by Eq.~(\ref{main:methods:eq:compressed-kde})}
    \State{Compute $v_d^\gamma$ by Eq.~(\ref{main:methods:eq:compressed-ped})}
    \EndFor
    \State{\textbf{return} $\{v_d^\gamma\}_{d=1}^D$}
  \end{algorithmic}
\end{algorithm}

%% file: sections/methods/quick-algo.tex
\subsection{Fast Algorithm by Pearson Divergence}
If we analyze the binary function
$b(\xv | \Xgp) \coloneqq \indic{\xv \in \Xgp} = \indic{f(\xv) \leq f^{\gamma^\prime}}$
instead of $f(\xv)$,
HPI can be efficiently computed
where $\gamma^\prime (< \gamma)$ is another quantile
to define the binary function in the local space $\Xg$.
First, we prove the following theorem:
\begin{theorem}
  Given the binary function $b(\xv|\Xgp)$ and
  the $\gamma^\prime$- and $\gamma$-set
  $\mathrm{PDFs}$ $p(\xv|\Xgp), p(\xv | \Xg)$
  where $\gamma^\prime < \gamma$,
  the local marginal variance of each dimension $d \in [D]$ is:
  \begin{equation}
    \begin{aligned}
      v_d^{\gamma}  = \biggl(
      \frac{\gamma^\prime}{\gamma}
      \biggr)^2
      ~\pe{p_d(\cdot | \Xgp)}{p_d(\cdot | \Xg)}.
    \end{aligned}
    \label{main:methods:eq:local-importance}
  \end{equation}
  \label{main:methods:theorem:local-importance}
\end{theorem}
The proof is provided in Appendix~\ref{appx:proofs:section:proof-of-local-hpi}
and higher orders of HPI can be computed by Eq.~(\ref{appendix:proofs:eq:general-local-importance})
in Appendix~\ref{appx:proofs:section:proof-of-local-hpi}.
Note that PED between the PDFs $p, q$
defined on $\X_d$ is computed as:
\begin{equation}
\begin{aligned}
  \pe{p}{q} \coloneqq \mathbb{E}_{x_d \sim q(x_d)}\biggl[
    \biggl(
      \frac{p(x_d)}{q(x_d)} - 1
    \biggr)^2
  \biggr].
\end{aligned}
\end{equation}
As we do not have the ground truth of
the marginal $\gamma^\prime$- and $\gamma$-set PDFs,
we replace them with KDEs.
The tricks of this computation are that
(1) the marginal $\gamma$-set PDF can be easily estimated by (1D) KDE
as follows and
(2) we only need to take the average in 1D space:
\begin{equation}
\begin{aligned}
  p_d(x_d | \D^\gamma) \coloneqq
  \frac{1}{\ceil{\gamma N}}\sum_{n=1}^{\ceil{\gamma N}}
  k(x_{n, d}, x_d).
\end{aligned}
\end{equation}
Note that $x_{n,d} \in \X_d$
is the $d$-th dimension of $\xv_n$
and $k$ is a kernel function.
Although the query of this function still requires $O(N)$,
the time complexity scales down to $O(n_d)$
where $n_d \in \mathbb{Z}_+$ is the number of
unique values in the $d$-th HP if we use the following compression:
\begin{equation}
\begin{aligned}
  p_d(x_d | \D^\gamma) =
  \frac{1}{\ceil{\gamma N}}\sum_{n=1}^{n_d}
  c_d^{(n)} k(x_d^{(n)}, x_d)
\end{aligned}
\label{main:methods:eq:compressed-kde}
\end{equation}
where $x_d^{(n)}$ is the $n$-th unique value in the $d$-th HP
and $c_d^{(n)}$ is the occurrences of this value in $\D^\gamma$.
Note that we discretize a continuous HP $x_d \in [L, R] (L < R)$ (if exists)
as $x_d \in \{L + n (R - L) /(n_d - 1)\}_{n=0}^{n_d - 1}$ to apply Eq.~(\ref{main:methods:eq:compressed-kde})
and the discretization error of marginal variances is bounded by $O(\frac{1}{n_d})$ under some assumptions.
Since we can avoid Monte-Carlo samplings with the discretization and the total time complexity is reduced to $O(N + n_d^2)$,
this is a trade-off;
see Proposition~\ref{appx:proofs:proposition:discretization-error}
in Appendix~\ref{appx:proofs:section:discretization-error} for more details.
Hence Eq.~(\ref{main:methods:eq:marginal-local-var}) is approximated as the following closed-form:
\begin{equation}
\begin{aligned}
  v_d^\gamma
  \simeq \biggl(\frac{\gamma^\prime}{\gamma}\biggr)^2
  \sum_{n=1}^{n_d} \frac{
    p_d(x_d^{(n)} | \D^\gamma)
  }{Z}
  \Biggl(
    \frac{
      p_d(x_d^{(n)} | \D^{\gamma^\prime})
      }{
        p_d(x_d^{(n)} | \D^\gamma)
      } - 1
  \Biggr)^2
\end{aligned}
\label{main:methods:eq:compressed-ped}
\end{equation}
where $Z \coloneqq \sum_{n=1}^{n_d} p_d(x_d^{(n)} | \D^\gamma)$
is a normalization constant and the time complexity of this computation is $O(n_d^2)$.
Algorithm~\ref{main:methods:alg:anova} shows the pseudocode for the local HPI computation.
Note that global HPI, whose computation is detailed in Appendix~\ref{appx:global-hpi:section}, can be computed by replacing $p_d(\cdot | \D^\gamma)$ with the uniform PDF in the $d$-th dimension $u_d(x_d)$ (or $p_d(\cdot | \D^\gamma)$ with $\gamma = 1$, i.e. $p_d(\cdot | \D)$, as discussed in Appendix~\ref{appx:use-cases:section:posthoc-hpo} when collecting $\D$ by a non-uniform sampler).
Figure~\ref{main:methods:fig:conceptual-ped-anova} presents an example of global HPI with our method on a 2D toy function.

\vspace{-1mm}

%% file: sections/validations/main.tex
\begin{figure}
  \centering
  \includegraphics[width=0.49\textwidth]{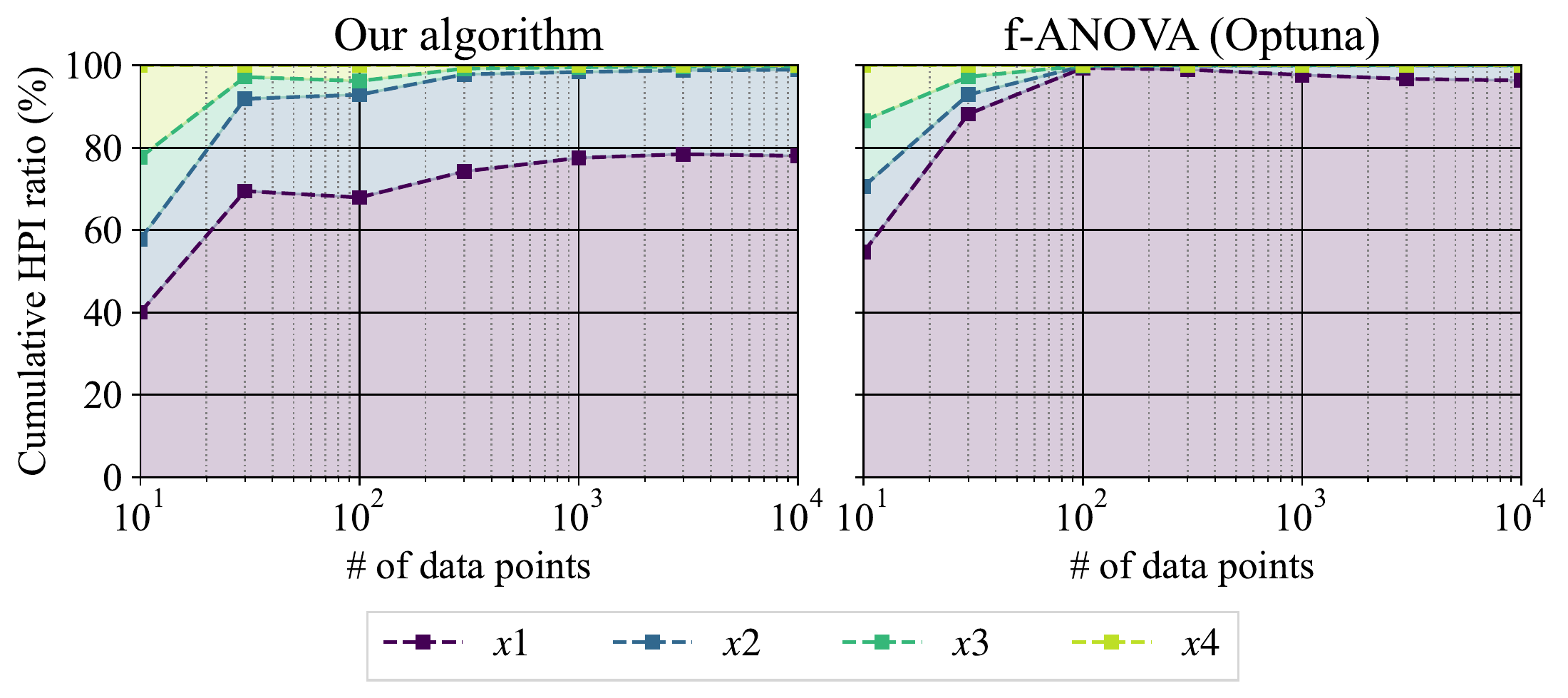}
  \vspace{-6mm}
  \caption{
    The comparison of global HPI between
    our algorithm (\textbf{Left}) and Optuna f-ANOVA (\textbf{Right}).
    HPI is averaged over $10$ runs with different random seeds.
    The horizontal axis shows the number of data points $N$
    and the vertical axis shows the cumulative HPI ratio.
    HPI ratio is computed by $v_d / \sum_{d^\prime=1}^D v_{d^\prime}$
    and the weak color band between each plot
    shows the HPI ratio of each HP.
  }
  \vspace{-3mm}
  \label{main:validation:fig:info-loss-benchmark}
\end{figure}

\section{Performance Validation}
\label{main:validation:section:performance-validation}

\input{sections/validations/setup.tex}

\begin{figure}[t]
  \centering
  \includegraphics[width=0.46\textwidth]{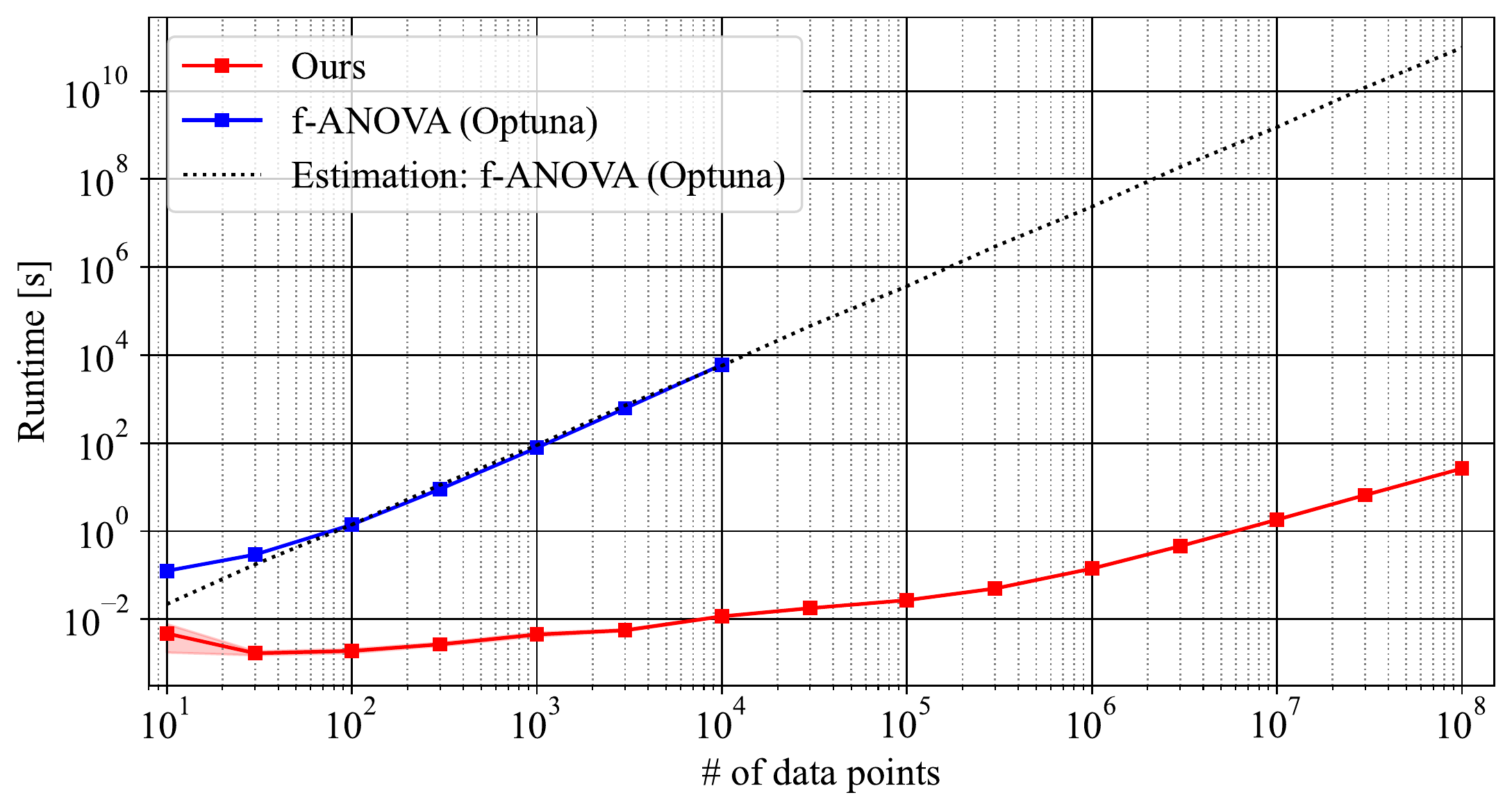}
  \vspace{-3mm}
  \caption{
    The benchmark of query speed of our method and f-ANOVA
    with respect to the number of data points $N$.
    Each setting was run with $10$ different seeds
    and the weak color bands show the standard error.
    As f-ANOVA requires much more computation,
    we estimated the evolution and provided the estimation
    by the black dotted line.
  }
  \vspace{-3mm}
  \label{main:validation:fig:runtime-benchmark}
\end{figure}

\input{sections/validations/scale-ignorance.tex}

\input{sections/validations/query-speed.tex}

\begin{figure*}[t]
  \centering
  \includegraphics[width=0.9\textwidth]{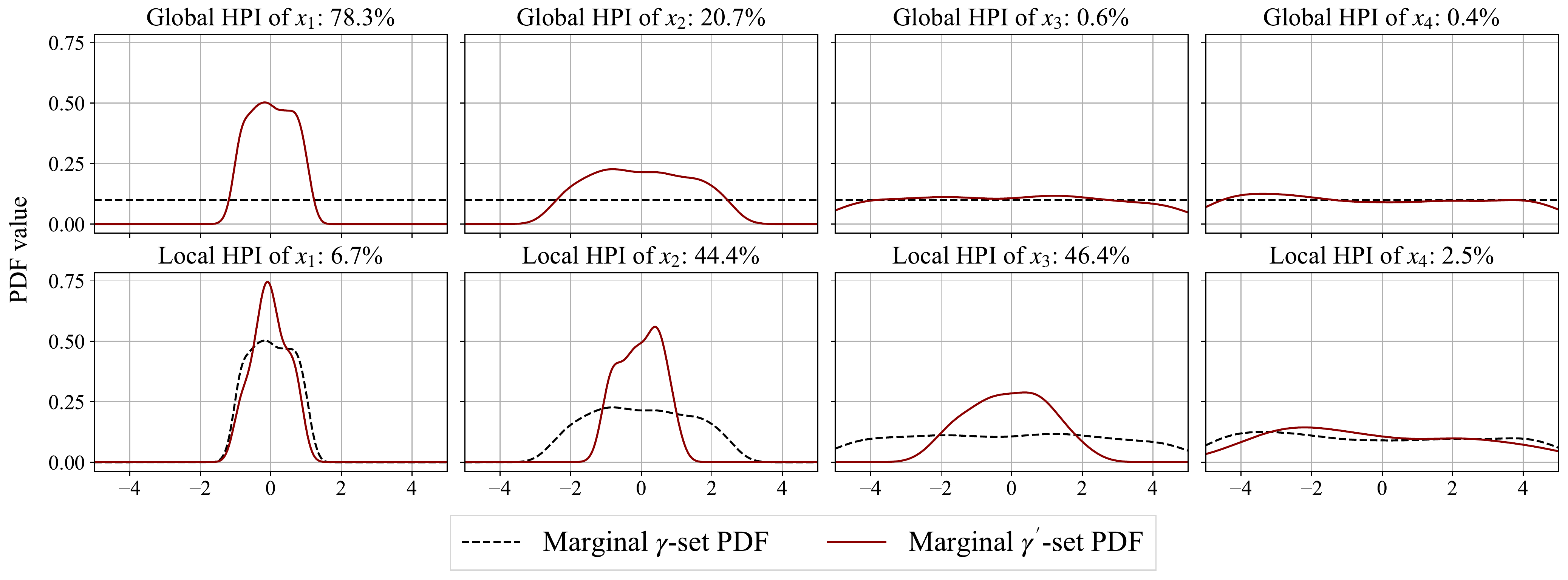}
  \vspace{-3mm}
  \caption{
    The validation of the local importance measure.
    The black dashed lines show the marginal $\gamma$-set PDFs
    and the red lines show the marginal $\gamma^\prime$-set PDFs for each dimension.
    The percentage (\textbf{HPI ratio}) was computed by
    $v_d / \sum_{d^\prime=1}^D v_{d^\prime}$.
    \textbf{Top row}:
    the plots of the uniform PDFs ($\gamma = 1$)
    and the marginal $\gamma^\prime = 0.1$-set PDFs.
    These are used to compute global HPI.
    \textbf{Bottom row}:
    the plots of
    the marginal $\gamma = 0.1$-set PDFs
    and the marginal $\gamma^\prime = 0.01$-set PDFs.
    Those are used to compute local HPI in the top-$10\%$ domain.
  }
  \vspace{-3mm}
  \label{main:validation:fig:local-validation}
\end{figure*}

\input{sections/validations/local-importance-check.tex}

%% file: sections/validations/setup.tex
\subsection{Setup}
In this section, we consistently use the following function:
\begin{equation}
  \begin{aligned}
    f(x_1, x_2, x_3, x_4) = \sum_{d=1}^4 w_d(x_d) \times x_d^2
  \end{aligned}
  \end{equation}
  where $x_d \in [-5, 5]$ for all $d \in \{1,2,3,4\}$
  and 
  the weights $w_d: \mathbb{R} \rightarrow \boldsymbol{W}$ follow:
  \begin{eqnarray}
    w_d(x) = \left\{
    \begin{array}{ll}
       W_{d - 1} & (|x| \geq 1) \\
       W_{d+2~\mathrm{mod}~4} & (\mathrm{otherwise})
    \end{array}
    \right..
  \end{eqnarray}
  and $\boldsymbol{W} \coloneqq \{W_d\}_{d=0}^3 = \{5^0, 5^{-1},  5^{-2},  5^{-3}\}$.
This function has different trends of HPI in
global and local spaces.
While the order of HPI is $x_1, x_2, x_3, x_4$ in the global space,
it is $x_2, x_3, x_4, x_1$ in the local space $\{\xv \in \X | \forall d \in [4], |x_d| < 1\}$.

In this experiment, we discretized the HPs with $n_d = 1001$
and all samples were drawn from the uniform distribution.
Furthermore, all experiments were run on the hardware with Intel Core i7-10700
and we used the f-ANOVA implementation with the default parameter setting
by Optuna~\footnote{https://github.com/optuna/optuna}.
Note that the Optuna implementation is based on \citewithname{hutter2014efficient}.

%% file: sections/validations/scale-ignorance.tex
\subsection{Effect of Scale Ignorance in Global HPI}
\label{main:validation:section:scale-ignorance}
Since PED-ANOVA uses $\indic{f(\xv) \leq f^\star}$ instead
of $f(\xv)$, it cannot capture scale information.
On the other hand, since our objective is to identify important HPs,
we would like to test if PED-ANOVA can identify important HPs.
In the experiment, we used $\gamma^\prime = 0.1$.
Figure~\ref{main:validation:fig:info-loss-benchmark}
shows the cumulative global HPI ratio of each method.
As seen in the figure, while both methods could identify
the most important HP $x_1$,
we can see the difference in the HPI of $x_2$.
While PED-ANOVA tells us $x_2$ has about $20\%$ of contribution
to achieve the top-$10\%$,
f-ANOVA tells us $x_2$ has about $3\%$ of contribution.
This difference comes from whether we ignore the scale of
the objective function or not.
Since f-ANOVA considers scale and it magnifies the contribution in the tail of the function, it dilutes the HPI of $x_2$, which has less weight in the tail.
Note that ``tail'' refers to the domains that
cause a lot of variations, yet not critical for the final result,
in the objective function $f$
and $|x_d| \geq 1$ is the tail in our case; more details in Appendix~\ref{appx:theoretical-details:section:scale-ignorance}.
On the other hand, the HPI of $x_1$ by PED-ANOVA is not strongly biased by the tail
due to the scale ignorance nature.
This leads to more importance in $x_2$.
Although our method loses scale information,
the ignorance of scale allows us to
abandon the information from the tail
and focus only on the information from the promising
domain, which is $\gamma^\prime$-set in our case.
Furthermore, this remarkable property makes
the meaning of HPI,
which is how important each HP is to achieve the top-$\gamma^\prime$
quantile,
very clear and practitioners can extract the nuance of each HP for specific local spaces.

%% file: sections/validations/query-speed.tex
\subsection{Query Speed}
As mentioned previously, one of the benefits of our method is the query speed,
and we would like to benchmark how quick our method is in this section.
In the experiment, we used $\gamma^\prime = 0.1$.
Figure~\ref{main:validation:fig:runtime-benchmark}
presents the runtime with respect to the number of data points.
While f-ANOVA requires more than a minute with $10^3$ data points
and more than a week with $10^5$ data points,
our method provides the results in a minute even with $10^8$ data points.
In Appendix~\ref{appx:experiments:section:sample-efficiency},
we test our method with higher dimensionality to see the number of data points required for convergence.

%% file: sections/validations/local-importance-check.tex
\subsection{Local Importance Measure}
Finally, we check if our method can successfully
identify important HPs in promising domains.
The objective function $f(\xv)$ is designed
so that while $x_1$ is important and $x_3$ is trivial in the global space,
$x_1$ is less important and $x_3$ is important in the local space.
The goal of this experiment is to check whether our method can
provide this insight.
In the experiment, we used $N = 10^4$.

Figure~\ref{main:validation:fig:local-validation} shows the results.
As discussed already, global HPI could identify
the order of HPI appropriately.
For local HPI, our method could tell us that
$x_2, x_3$ are the most important HPs
in the local space
and $x_1$ is less important as expected.
Note that since the $\gamma = 0.1$-set
already narrows down the domain of $x_2$, but not $x_3$,
this dilutes the HPI of $x_2$ and increases the HPI of $x_3$.
Prior works cannot provide this interpretation
as discussed in Appendix~\ref{appx:theoretical-details:section:difference-from-prior-works}.

%% file: sections/experiments/main.tex
\addtolength{\tabcolsep}{0pt}
\begin{table}[t]
  \begin{center}
    \caption{
      HPI of \texttt{CIFAR10} in JAHS-Bench-201.
      The ratio of HPI by percentage (\textbf{HPI ratio}) computed by $v_d/\sum_{d^\prime=1}^{D} v_{d^\prime}$.
      The top-$3$ HPs are bolded.
      \textbf{Cols.~1,3,5} (Original):
      HPI by f-ANOVA on $g(\xv) \coloneqq \min(f(\xv), f^{\gamma^\prime})$.
      \textbf{Cols.~2,4,6} (Ours):
      HPI by PED-ANOVA.
    }
    \vspace{-3mm}
    \label{main:experiments:tab:jahs-hpi-list}
    \makebox[0.49 \textwidth][c]{       
      \resizebox{0.49 \textwidth}{!}{
        \begin{tabular}{lc|cc|cc|c}
          & \multicolumn{6}{c}{HPI ratio ($\%$)} \\
          \toprule
          Hyperparameter      & Normal         & \multicolumn{2}{c|}{Global 0.1} & \multicolumn{2}{c|}{Global 0.01} & Local                                            \\
                              & Original       & Ours                           & Original                        & Ours           & Original       & Ours           \\
          \midrule
          Learning rate       & 1.36           & \textbf{9.11}                  & \textbf{10.20}                  & 6.62           & 3.59           & 4.09           \\
          Weight decay        & 0.96           & 2.19                           & 0.68                            & 2.56           & 0.31           & 3.00           \\
          Activation function & 0.01           & 0.12                           & 0.21                            & 0.26           & 0.41           & 0.40           \\
          TrivialAugment      & 0.00           & 4.33                           & \textbf{3.83}                   & \textbf{13.22} & \textbf{8.27}  & \textbf{28.33} \\
          \midrule
          Depth multiplier    & 0.06           & 0.66                           & 0.58                            & 2.47           & 0.63           & 6.90           \\
          Width multiplier    & 1.60           & \textbf{60.22}                 & \textbf{73.59}                  & \textbf{35.26} & \textbf{71.75} & 9.07           \\
          \midrule
          Operation 1 (Op.1)  & \textbf{11.86} & \textbf{6.65}                  & 3.45                            & \textbf{11.95} & 3.81           & \textbf{13.38} \\
          Operation 2 (Op.2)  & 4.04           & 2.36                           & 1.42                            & 5.00           & 2.51           & 6.97           \\
          Operation 3 (Op.3)  & \textbf{64.73} & 5.63                           & 1.14                            & 5.25           & 1.73           & 5.50           \\
          Operation 4 (Op.4)  & 0.09           & 0.84                           & 0.83                            & 1.62           & 1.09           & 2.09           \\
          Operation 5 (Op.5)  & 4.00           & 2.19                           & 1.04                            & 4.72           & 1.29           & 6.76           \\
          Operation 6 (Op.6)  & \textbf{11.29} & 5.71                           & 3.02                            & 11.06          & \textbf{4.61}  & \textbf{13.52} \\
          \bottomrule
        \end{tabular}
      }
    }
  \end{center}
  \vspace{-4mm}
\end{table}
\addtolength{\tabcolsep}{0pt}

\begin{figure*}[t]
  \centering
  \includegraphics[width=0.9\textwidth]{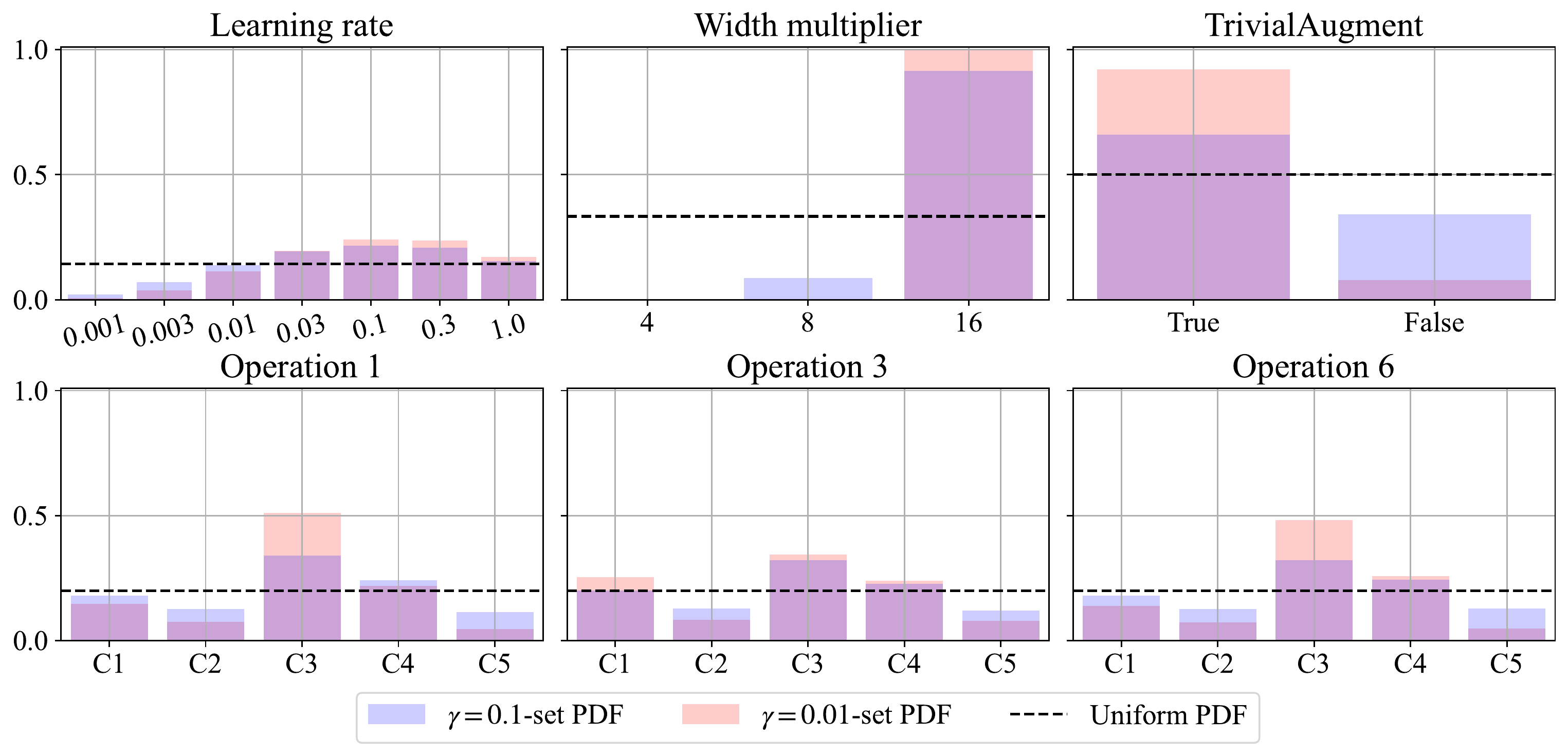}
  \vspace{-3mm}
  \caption{
    The distributions of important HPs of \texttt{CIFAR10} in JAHS-Bench-201.
    The red shadows show the $\gamma=0.01$-set PDFs,
    the blue shadows show the $\gamma=0.1$-set PDFs, and
    the black dashed lines show the uniform PDFs.
    PED between a black line and a blue shadow
    is \texttt{Global 0.1},
    PED between a black line and a red shadow
    is \texttt{Global 0.01},
    and PED between a red shadow and a blue shadow
    is \texttt{Local}
    in Table~\ref{main:experiments:tab:jahs-hpi-list}.
    The titles for each figure show the names of each HP and the details of HPs are available in
    Appendix~\ref{appx:experiments:section:dataset-detail}.
    Notice that C1 -- C5 correspond to the order of Table~\ref{appx:experiments:tab:jahs-search-space}
    and the overlap between the red and the blue shadows looks purple although they are separated shadows.
  }
  \vspace{-3mm}
  \label{main:experiments:fig:jahs-histogram}
\end{figure*}

\vspace{-1mm}
\section{Real-World Usecase by JAHS-Bench-201}
\label{main:experiments:section}

\input{sections/experiments/setup.tex}

\input{sections/experiments/analysis.tex}

%% file: sections/experiments/setup.tex
\subsection{Setup}
\label{main:experiments:section:setup}
In order to further verify our proposed algorithm against a real-world application,
we applied PED-ANOVA to analyze the search space of JAHS-Bench-201~\cite{bansal2022jahs},
which is a surrogate benchmark for HPO and has a very large search space in the context of extant HPO benchmarks.
We constructed the dataset $\D$ in Algorithm~\ref{main:methods:alg:anova} by querying JAHS-Bench-201 for the validation accuracy, i.e. $f(\xv)$, of $N$ lattice points,
where $N = 41{,}343{,}750$,
generated by discretizing the JAHS-Bench-201 search space
(see  Table~\ref{appx:experiments:tab:jahs-search-space} of Appendix~\ref{appx:experiments:section:dataset-detail}).
Although JAHS-Bench-201 can be queried for model performance metrics on 3 different image classification datasets,
for the sake of brevity, we discuss only the experiments performed on \texttt{CIFAR10} here
and include the results on the other datasets in Appendix~\ref{appx:experiments:section:additional-results}.
Due to the computational complexity of f-ANOVA, we could use only $10^4$ data points for it, in contrast to PED-ANOVA,
and calculated the mean of HPI over $10$ independent runs.
Since the surrogate models in JAHS-Bench-201 are trained XGBoost models
and XGBoost's outputs are deterministic, we query each lattice point only once.
In the analysis, we would like to answer the following research questions (RQs):
\begin{enumerate}
  \addtolength{\itemindent}{5mm}
  \vspace{-0.0mm}
  \item [\textbf{RQ1}:] Does global HPI of our method provide the same important HPs as f-ANOVA with Eq.~(\ref{main:background:eq:frank-anova-transformation})?
  \vspace{-1mm}
  \item [\textbf{RQ2}:] Is the scale ignorance necessary for matching the intuition of achieving the top-$\gamma^\prime$ quantile?
  \vspace{-1mm}
  \item [\textbf{RQ3}:] Does local HPI help detect potentially important HPs or trivial HPs?
  \vspace{-0.0mm}
\end{enumerate}
In order to answer RQs,
we provide Table~\ref{main:experiments:tab:jahs-hpi-list}
with the HPI of each HP in \texttt{CIFAR10} of JAHS-Bench-201
and Figure~\ref{main:experiments:fig:jahs-histogram} to
visualize the $\gamma=0.1$- and $\gamma=0.01$-set PDFs
as the blue and the red shadows, respectively.
Strictly speaking,
discrete probability distributions are not PDFs due to discrete space;
however, we use the term $\gamma$-set PDF for the sake of consistency.
We applied f-ANOVA to $f(\xv)$ (\texttt{Normal}),
$\min(f(\xv), f^{\gamma^\prime = 0.1})$ (\texttt{Global 0.1}),
and $\min(f(\xv), f^{\gamma^\prime = 0.01})$ (\texttt{Global 0.01}), and
PED-ANOVA with $\gamma^\prime = 0.1$ (\texttt{Global 0.1}),
$\gamma^\prime = 0.01$ (\texttt{Global 0.01}), and
$\gamma = 0.1, \gamma^\prime = 0.01$ (\texttt{Local}).
Recall that \texttt{Global 0.1} and \texttt{Global 0.01} for f-ANOVA are based on Eq.~(\ref{main:background:eq:frank-anova-transformation})
and $f^{\gamma^\prime = 0.1}$ means $f_{\ceil{|\D|/10}}$ given a dataset sorted by $f_n$.
Although we used the uniform PDF to compute global HPI in this experiment,
practitoners should use $p_d(\cdot | \D) (\gamma = 1)$ instead of the uniform PDF
for the post-hoc analysis of HPO when using a non-uniform sampler (e.g. Bayesian optimization) to remove sampling bias 
as discussed in Appendix~\ref{appx:use-cases:section:posthoc-hpo}.

%% file: sections/experiments/analysis.tex
\subsection{Analysis \& Interpretation}
To answer RQ1, we compare the column \texttt{(Global 0.1, Ours)} to \texttt{(Global 0.1, Original)} and
the column \texttt{(Global 0.01, Ours)} to \texttt{(Global 0.01, Original)} in Table~\ref{main:experiments:tab:jahs-hpi-list}.
We observe that both PED-ANOVA and f-ANOVA indicated the same top-2 important HPs although the $3^{\mathrm{rd}}$-best HPs slightly varied.
This result further verifies the validation in Section~\ref{main:validation:section:scale-ignorance}.

To answer RQ2, we discuss the results of \texttt{(Global 0.1, Ours)} and \texttt{(Global 0.01, Ours)} in the context of \texttt{(Normal, Original)}
to assess the impact that the tail of $f(\xv)$ discussed in Section~\ref{main:validation:section:scale-ignorance} has on f-ANOVA.
The most important takeaway from this comparison is the misclassification of \texttt{Op.3} as the most important and
of \ta\ as the least important HP to optimize over by the original f-ANOVA in the global setting.
As can be verified by looking at the $\gamma$-set PDFs, even for $\gamma=0.01$,
\texttt{Op.3}'s values are distributed very evenly even when \ta\ and \wm\ have already shown convergence.
This clearly indicates that \texttt{Op.3} is not very important to optimize for achieving the top-$1\%$ performance and
may or may not become relevant in even higher quantile regimes.
At the same time, both columns' values agree on the importance of \texttt{Op.1} and \texttt{Op.6}.
Therefore, to answer RQ2, scale invariance indeed helps to successfully identify HPI for HPs that would have been misclassified by the \texttt{(Normal, Original)} setting.

Finally, for RQ3, we compare the column \texttt{(Global 0.01, Ours)} to \texttt{(Local, Ours)}. 
We observe that the HPI of \wm\ drops sharply from the \texttt{Global 0.01} setting to the \texttt{Local} setting.
Simultaneously, the HPI of \ta\ increases sharply across the same.
This suggests that optimizing \wm\ is no longer important when moving from the top-$10\%$ to the top-$1\%$ performance
but optimizing \ta\ is very important.
The reason behind this change becomes clear when we observe the change in $\gamma$-set PDFs of the two HPs in Figure~\ref{main:experiments:fig:jahs-histogram}.
Both the $\gamma$-set PDFs for \wm\ are sharply peaked at $16$,
indicating that no further optimization is needed on \wm.
However, the $\gamma$-set PDFs for \ta\ only start peaking at the value \texttt{True} for the $\gamma=0.01$-set PDF.
This clearly demonstrates that local HPI is necessary for deriving the correct interpretation in the top-$\gamma^\prime$ quantiles,
since \texttt{(Local, Ours)} successfully identifies the relative importance of optimizing the two HPs.
Last but not least, if both global and local HPI with wished quantiles $\gamma, \gamma^\prime$ exhibits low values,
removing such HPs, e.g. \texttt{Activation function}, is expected to have a less negative impact although it is insecure to remove HPs, e.g. \texttt{Op.1}, only by looking at global HPI.

%% file: sections/conclusions.tex
\section{Conclusions}
In this paper, we reformulated
f-ANOVA for local HPI
and introduced the fast algorithm to compute
local HPI by PED.
In the series of experiments on a toy function,
we confirmed that our method can quantify
both global and local HPI appropriately,
and efficiently compute HPI in a second with $10^5$ data points
while the prior work takes several days.
In the analysis of JAHS-Bench-201,
we provided a concrete example of how
to use our method on benchmark datasets
and showed that only using global HPI could be misleading.
Due to the space limit, we defer a discussion of practical
usecases and limitations of our method
to Appendix~\ref{appx:use-cases:section}.
Our implementation is available at \url{https://github.com/nabenabe0928/local-anova/}.

%% file: sections/acknowledgements.tex
\section*{Acknowledgments}
The authors appreciate the valuable contributions of the anonymous reviewers.
Robert Bosch GmbH is acknowledged for financial support.
The authors also acknowledge funding by European Research Council (ERC) Consolidator Grant ``Deep Learning 2.0'' (grant no.\ 101045765).
Views and opinions expressed are however those of the authors only and do not necessarily reflect those of the European Union or the ERC.
Neither the European Union nor the ERC can be held responsible for them.

\begin{center}
  \includegraphics[width=0.3\textwidth]{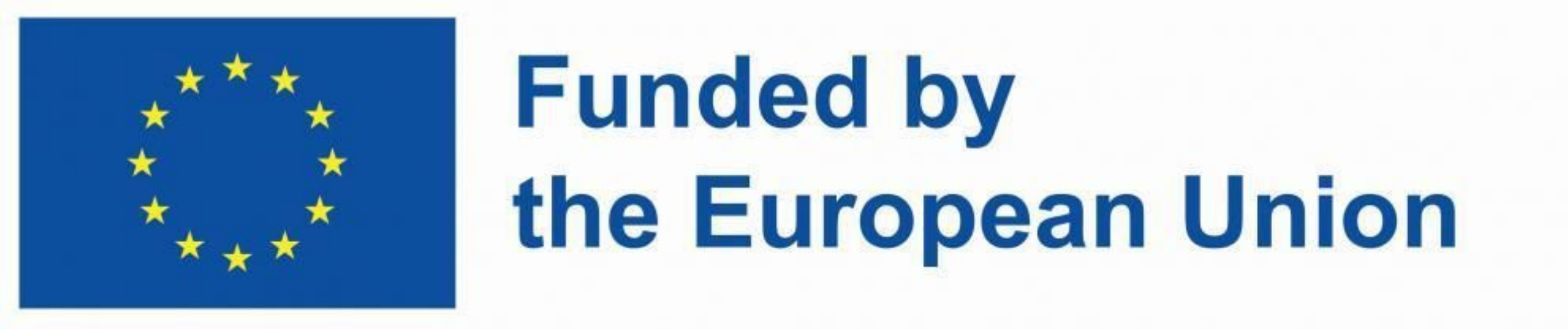}
\end{center}

%% file: settings/appendix-information.tex
\ifappendix
\clearpage
\appendix
\input{appendices/appendix-full-contents.tex}

\else

\customlabel{appx:background:section:general-anova}{A.3}
\customlabel{appx:background:eq:general-marginal-mean}{19}
\customlabel{appx:background:eq:general-zero-centered-marginal-mean}{20}
\customlabel{appx:background:eq:general-marginal-var}{21}
\customlabel{appx:theoretical-details:section}{B}
\customlabel{appx:theoretical-details:fig:ped-conceptual}{9}
\customlabel{appx:global-hpi:section}{B.2}
\customlabel{appx:theoretical-details:eq:global-hpi}{27}
\customlabel{appx:theoretical-details:proposition:global-hpi}{1}
\customlabel{appx:theoretical-details:section:local-hpi}{B.3}
\customlabel{appx:theoretical-details:eq:local-hpi}{28}
\customlabel{appx:theoretical-details:theorem:local-hpi}{2}
\customlabel{appx:theoretical-details:fig:toy-example-viz}{10}
\customlabel{appx:theoretical-details:fig:length-vs-hpi}{11}
\customlabel{appx:theoretical-details:section:difference-from-prior-works}{B.4}
\customlabel{appx:theoretical-details:section:analysis-clipped}{B.4.1}
\customlabel{appx:theoretical-details:eq:local-hpi-by-hutter}{29}
\customlabel{appx:theoretical-details:fig:global-eg-on-toy}{12a}
\customlabel{sub@appx:theoretical-details:fig:global-eg-on-toy}{(a)}
\customlabel{appx:theoretical-details:fig:local-eg-on-toy}{12b}
\customlabel{sub@appx:theoretical-details:fig:local-eg-on-toy}{(b)}
\customlabel{appx:theoretical-details:fig:global-vs-local}{12}
\customlabel{appx:theoretical-details:section:analysis-on-func}{B.4.2}
\customlabel{appx:theoretical-details:eq:toy-func}{31}
\customlabel{appx:theoretical-details:section:splits}{B.5}
\customlabel{appx:theoretical-details:section:scale-ignorance}{B.6}
\customlabel{appx:theoretical-details:fig:scale-ignorance}{13}
\customlabel{appx:proofs:section}{C}
\customlabel{appx:proofs:section:assumptions}{C.1}
\customlabel{appx:proofs:section:proof-of-global-hpi}{C.2}
\customlabel{appx:proofs:eq:def-of-local-volume}{41}
\customlabel{appx:proofs:eq:binary-expectation-is-gamma}{42}
\customlabel{appx:proofs:lemma:integral-of-marginal-dist}{1}
\customlabel{appx:proofs:eq:general-global-hpi}{47}
\customlabel{appx:proofs:section:proof-of-local-hpi}{C.3}
\customlabel{appx:proofs:eq:marginal-mean-is-density-ratio}{49}
\customlabel{appx:proofs:eq:binary-expectation-is-gamma-ratio}{50}
\customlabel{appendix:proofs:eq:general-local-importance}{52}
\customlabel{appx:proofs:section:discretization-error}{C.4}
\customlabel{appx:proofs:proposition:discretization-error}{2}
\customlabel{appx:experiments:section}{D}
\customlabel{appx:experiments:section:dataset-detail}{D.1}
\customlabel{appx:experiments:section:sample-efficiency}{D.2}
\customlabel{appx:experiments:section:additional-results}{D.3}
\customlabel{appx:experiments:tab:jahs-search-space}{2}
\customlabel{appx:experiments:fig:datasets-performance-dist}{14}
\customlabel{appx:use-cases:section}{E}
\customlabel{appx:experiments:tab:jahs-hpi-list}{3}
\customlabel{appx:experiments:fig:sample-efficiency}{15}
\customlabel{appx:use-cases:section:posthoc-hpo}{E.2}
\customlabel{appx:experiments:fig:jahs-histogram}{16}
\customlabel{appx:use-cases:section:space-reduction}{E.3}

\fi

%% file: appendices/appendix-full-contents.tex
\input{appendices/background/main.tex}

\input{appendices/theoretical-details/main.tex}

\input{appendices/proofs/main.tex}

\input{appendices/experiments/main.tex}

\input{appendices/use-cases.tex}

\bibliographystyleappx{bib-style}
\bibliographyappx{ref}

%% file: appendices/background/main.tex
\section{Background}
In the main paper, we discussed f-ANOVA
for one-dimensional effects for simplicity
and used concise notations.
In this section, we provide the generalized f-ANOVA
and discuss more strictly and precisely.

\input{appendices/background/notations.tex}

\input{appendices/background/lebesgue-integral.tex}

\input{appendices/background/general-anova.tex}

%% file: appendices/background/notations.tex
\subsection{Notations}

In Appendix, we use the following notations for simplicity:
\begin{itemize}
  \item $[D] \coloneqq \{1, 2, \dots, D\}$:
  a set of integers from $1$ to $D$, 
  \item $s \subseteq [D]$: a subset of $[D]$,
  \item $\Pcal_s \coloneqq 2^s$: the power set of $s$,
  \item $\Pcal_{-s} \coloneqq 2^s \setminus \{s\}$:
  the power set of $s$ without $s$ itself,
  \item $\Scal_{s} \coloneqq \{s\}$:
  a singleton only with $s \subseteq [D]$,
  \item $\Scal_{d} \coloneqq \Scal_{\{d\}}$:
  a singleton only with $\{d\}$,
  \item $\Pcal_d \coloneqq \Pcal_{\Scal_d} = \{\emptyset, \{d\}\}$:
  the power set of $\{d\}$,
  \item $\X \subseteq \mathbb{R}^D$: the search space of the function $f$,
  \item $\X_d \subseteq \mathbb{R}$: the domain of the $d$-th HP,
  \item $\X_s \subseteq \mathbb{R}^{|s|}$: the subspace of $\X$ built with the dimensions in $s$,
  \item $\X_{-s} \subseteq \mathbb{R}^{D - |s|}$: the complementary space of $\X_s$ on $\X$,
  \item $\xv_s \in \X_s$: a vector in $\X_s$,
  \item $\mu$: the Lebesgue measure,
  \item $\B_d$: the Borel body over $\mathbb{R}^d$,
  \item $\B_{\X_s} \coloneqq \B_{|s|} \cap \X_s$: the Borel body over $\X_s$,
  \item $b(\xv | A) \coloneqq \indic{\xv \in A}$: a binary function that returns $1$ if $\xv \in A$,
  \item $u(A) \coloneqq \frac{1}{\mu(A)}$:
  the uniform PDF defined on $A \subseteq \mathbb{R}^d$,
  \item $f(\xv | \xv_s)$:
  a function $f(\xv)$ with the dimensions in $s \subseteq [D]$
  are fixed to $\xv_s$,
  \item $f_s(\xv_s | S)$: 
  the marginal function of $f$ defined on $\X_s$
  that considers the interaction effects of dimensions
  belonging to $S \subseteq \Pcal_s$,
\end{itemize}
where $d \in [D]$.
Notice that the notations in the appendix
are slightly different from those in the main paper
to be more precise.

%% file: appendices/background/lebesgue-integral.tex
\subsection{Lebesgue Integral}
We first note that the definition of the Lebesgue measure $\mu$
changes based on the input for simplicity if not specified although this is abuse of notation.
For example, when we take $A \in \B_D \cap \X$ as an input,
$\mu$ is defined on $\X$
and when we take $A \in \B_{|s|} \cap \X_s$ as an input,
$\mu$ is defined on $\X_s$.
In this paper, we consistently use the Lebesgue integral $\int_{\xv \in \X} g(\xv)\mu(d\xv)$
instead of the Riemann integral $\int_{\xv \in \X} g(\xv)d\xv$;
however, $\int_{\xv \in \X} g(\xv)d\xv = \int_{\xv \in \X} g(\xv)\mu(d\xv)$
holds if $g(\xv)$ is Riemann integrable.
The only reason why we use the Lebesgue integral is that
some functions in our discussion cannot be handled by the Riemann integral,
and thus we encourage readers to replace $\mu(d\xv)$ with $d\xv$
if they are not familiar with the Lebesgue integral.

%% file: appendices/background/general-anova.tex
\subsection{Generalized f-ANOVA}
\label{appx:background:section:general-anova}
Suppose we would like to quantify HPI of
a function $f(\xv)$ defined on $\X$, then global HPI~\citeappx{hooker2007generalized} requires:
\begin{enumerate}
  \item \textbf{Marginal mean}:
  \begin{equation}
  \begin{aligned}
    f_s(\xv_s | \Pcal_s) \coloneqq \int_{\xv_{-s} \in \X_{-s}} f(\xv | \xv_s)\frac{\mu(d\xv_{-s})}{\mu(\X_{-s})}, \\
  \end{aligned}
  \label{appx:background:eq:general-marginal-mean}
  \end{equation}
  \item \textbf{Zero-centered marginal mean}:
  \begin{equation}
  \begin{aligned}
    f_s(\xv_s | \Scal_s) \coloneqq f_s(\xv_s | \Pcal_s) - \sum_{s^\prime \in \Pcal_{-s}} f_{s^\prime}(\xv_{s^\prime} | \Scal_{s^\prime}), \\
  \end{aligned}
  \label{appx:background:eq:general-zero-centered-marginal-mean}
  \end{equation}
  \item \textbf{Marginal variance}:
  \begin{equation}
  \begin{aligned}
    v_s \coloneqq \int_{\xv_s \in \X_s} f_s(\xv_s | \Scal_{s})^2 \frac{\mu(d\xv_{s})}{\mu(\X_{s})},
  \end{aligned}
  \label{appx:background:eq:general-marginal-var}
  \end{equation}
\end{enumerate}
where $f_s(\xv_s|\cdot)$ is a mean function defined on the subspace $\X_s$.
For the case of $s = \emptyset$,
as $f_{\emptyset}$ is the mean of $f(\xv)$ on $\X$,
$v_0 \coloneqq v_{\emptyset}$ becomes \emph{global variance}
and it is computed as:
\begin{equation}
\begin{aligned}
  v_0 = \int_{\xv \in \X}
  (
    \underbrace{m}_{
      ~~~~~~~\mathclap{\coloneqq f_{\emptyset}(\cdot | \emptyset)}
    }
    - f(\xv))^2 \frac{\mu(d\xv)}{\mu(\X)}.
\end{aligned}
\end{equation}
Then global HPI
of a set of dimensions $s$
is defined as the fraction 
$v_s/v_0$
as it is guaranteed that:
\begin{equation}
\begin{aligned}
  v_0 =
  \sum_{s \in \Pcal_{[D]} \setminus \Scal_{\emptyset}} v_s.
\end{aligned}
\end{equation}
For more details, see 
``variance decomposition'' in \citewithname{hooker2007generalized}.
Note that we consistently use $f_d$ as $f_{\Scal_d}$
and $v_d$ as $v_{\Scal_d}$.

%% file: appendices/theoretical-details/main.tex
\section{Theoretical Details of PED-ANOVA}
\label{appx:theoretical-details:section}

\input{appendices/theoretical-details/preliminaries.tex}

\begin{figure}[t]
  \centering
  \includegraphics[width=0.46\textwidth]{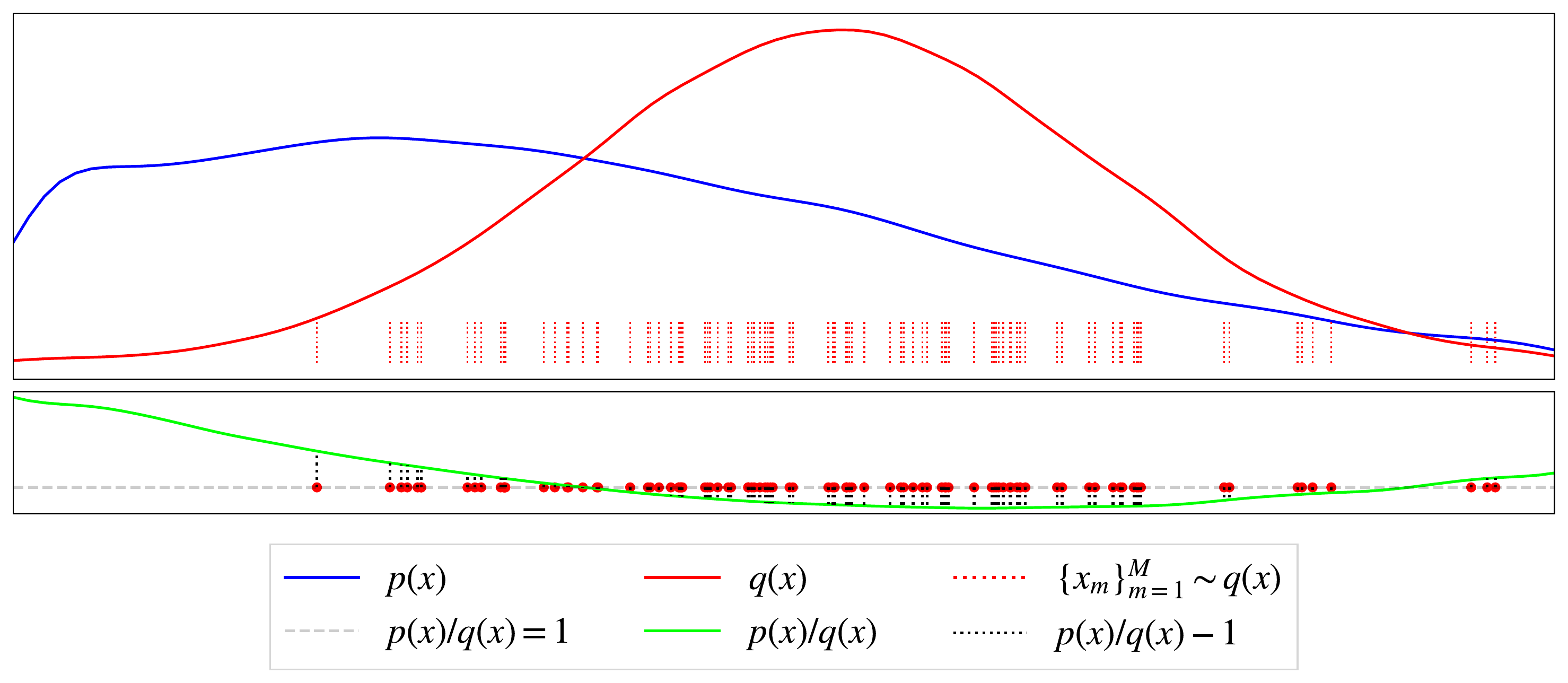}
  \vspace{-2mm}
  \caption{
    The conceptual visualization of Pearson divergence (PED).
    \textbf{Top}:
    the two PDFs (red, blue lines) of which we would like to measure PED.
    \textbf{Bottom}:
    the density ratio of the two PDFs (green line)
    and the baseline (gray dashed line),
    which is $p(x)/q(x) = 1$.
    PED measures the squared average of $p(x)/q(x) - 1$ (the black dotted lines)
    over the samples from $q(x)$
    (red dots) if the density ratio is defined over a continuous space;
    otherwise we can compute PED by a closed-form.
  }
  \vspace{-2mm}
  \label{appx:theoretical-details:fig:ped-conceptual}
\end{figure}

\input{appendices/theoretical-details/global-hpi.tex}

\input{appendices/theoretical-details/local-hpi.tex}

\begin{figure}[t]
  \centering
  \includegraphics[width=0.49\textwidth]{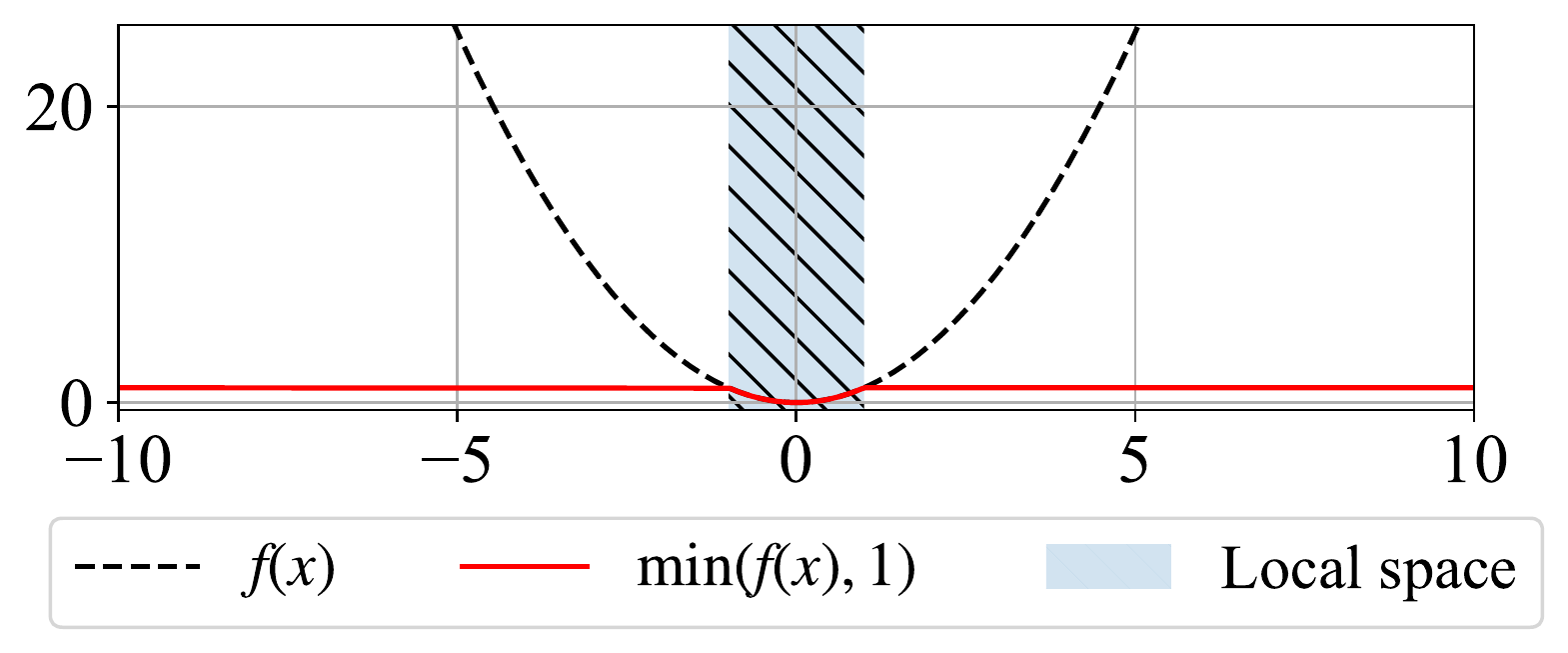}
  \caption{
    The visualization of the clipped function
    $g(x) = \min(f(x), 1) = \min(x^2, 1)$.
    The black dotted line is $f(x)$
    and the red solid line is $g(x)$.
    While global HPI computes the variance over the whole domain
    $[-10, 10]$, local HPI (\textbf{our proposition})
    computes the variance only over the blue-shaded domain.
  }
  \label{appx:theoretical-details:fig:toy-example-viz}
\end{figure}

\begin{figure}[t]
  \centering
  \includegraphics[width=0.49\textwidth]{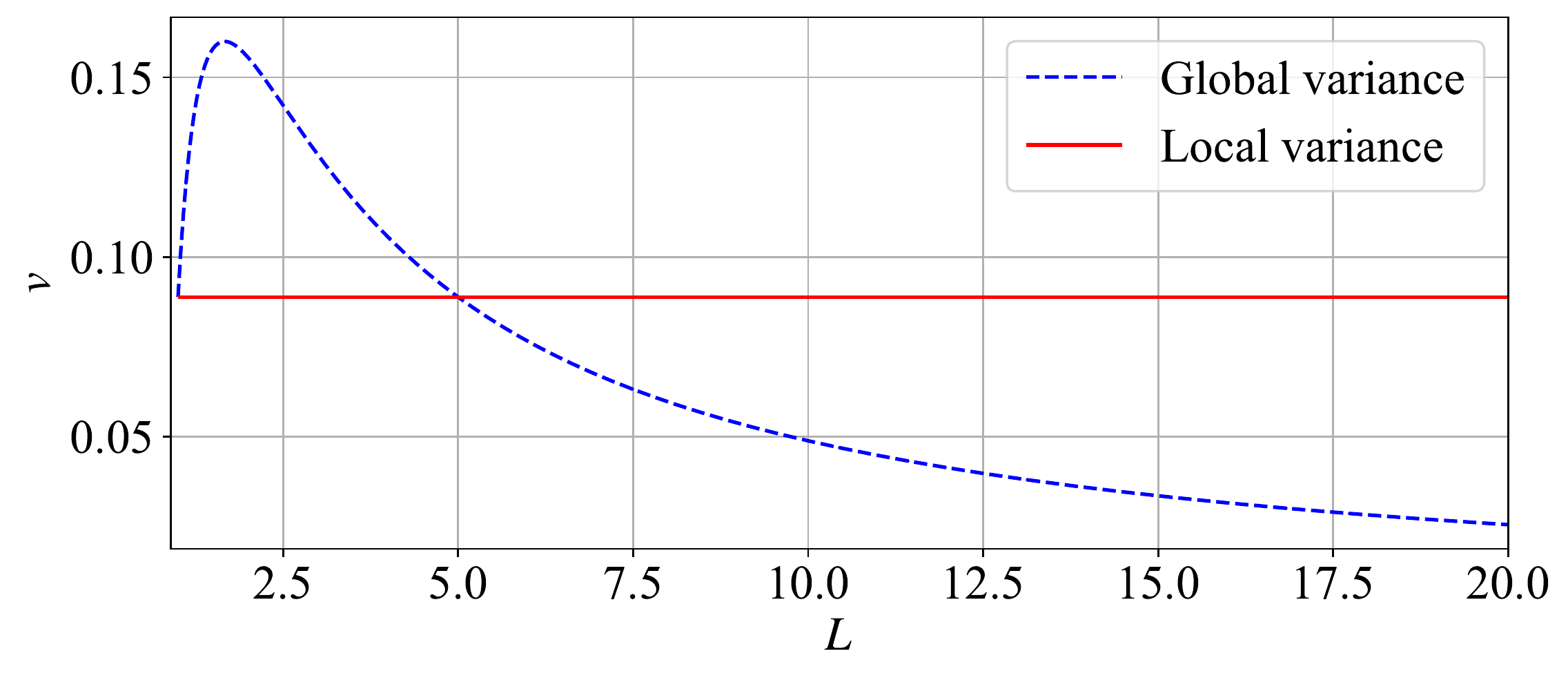}
  \caption{
    The plots of global/local variances of $g(x)$.
    The horizontal axis is the domain size $L$
    and the vertical axis is the variance of $g(x)$.
    The blue dashed line shows the global variance
    and the red solid line shows the local variance.
    Only the global variance changes depending on
    the search space design.
  }
  \label{appx:theoretical-details:fig:length-vs-hpi}
\end{figure}

\begin{figure*}[t]
  \begin{center}
    \subfloat[\sloppy Global space\label{appx:theoretical-details:fig:global-eg-on-toy}]{
      \includegraphics[width=0.4\textwidth]{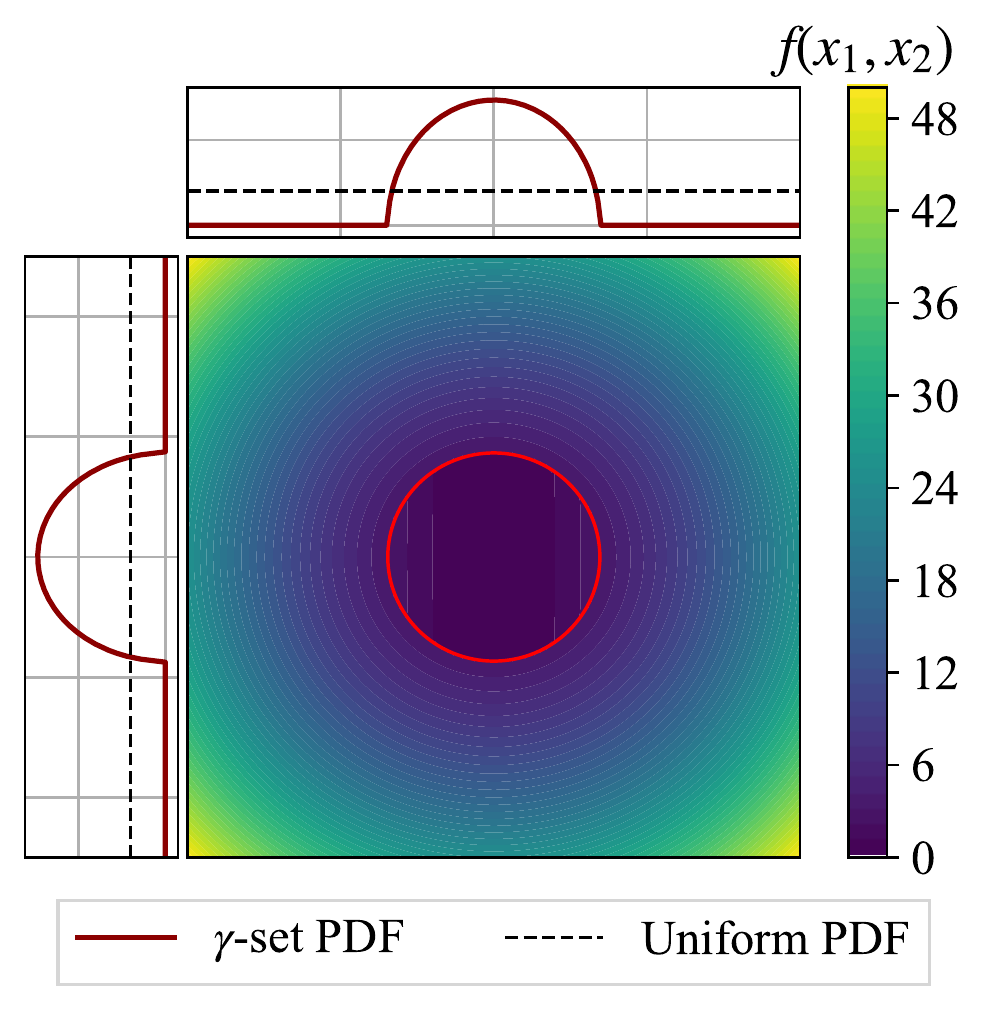}
    }
    \subfloat[\sloppy Local space\label{appx:theoretical-details:fig:local-eg-on-toy}]{
      \includegraphics[width=0.4\textwidth]{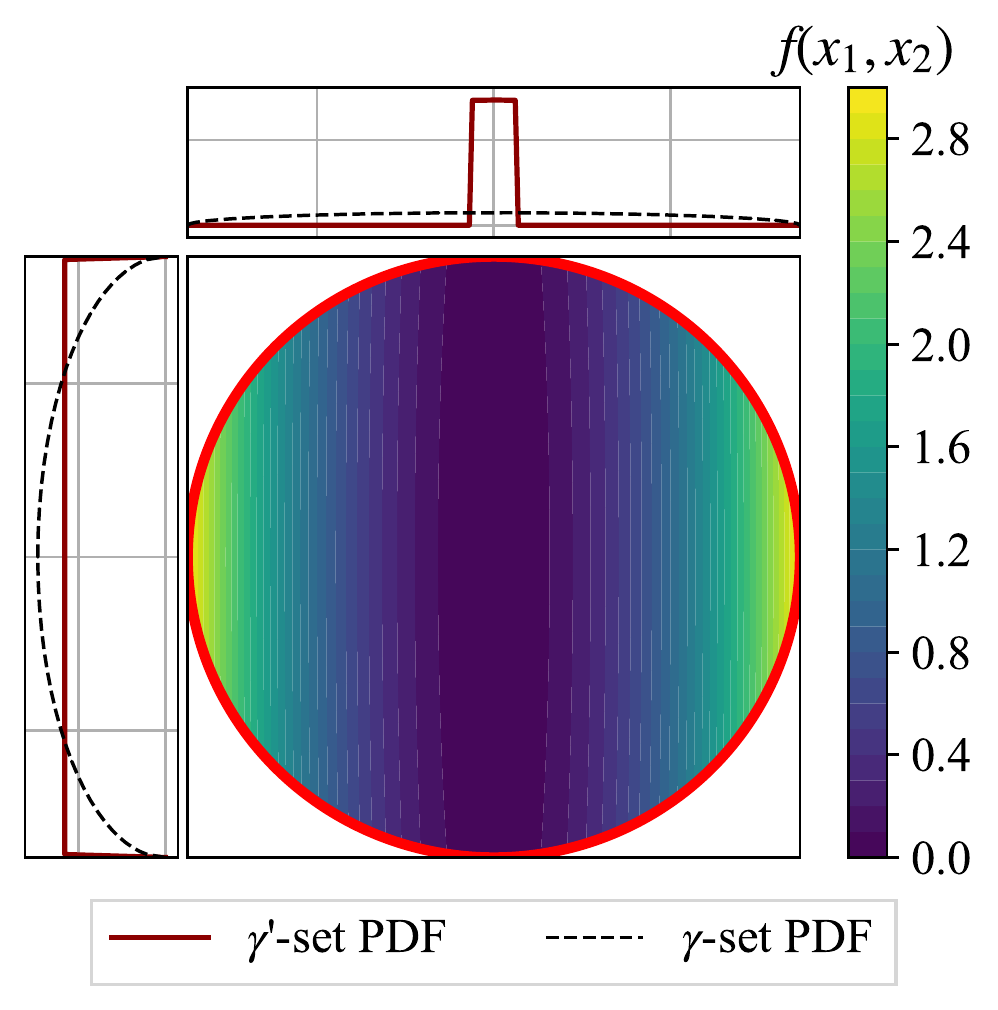}
    }
    \caption{
      The normalized contour plots of the function (darker is better)
      in Eq.~(\ref{appx:theoretical-details:eq:toy-func})
      in both global and local spaces.
      The function is defined on $[-5, 5] \times [-5, 5]$
      and the local space is defined as $f(x_1, x_2) < f^\gamma = 3$,
      which is $x_1^2 + x_2^2 < 3$.
      The red circle on the global space represents the local space.
      The red and black-dashed plots on each side of the figures show
      the marginal $\gamma$-set PDFs
      and the uniform PDFs in each space.
      \textbf{Left}:
      the visualization of the global space.
      Since the marginal $\gamma$-set ($\gamma = (\sqrt{3})^2 \pi / 100 \simeq 0.1$)
      PDFs for each dimension
      have the same shape, both dimensions are equally important.
      \textbf{Right}:
      the visualization of the local space.
      We took $\gamma^\prime = 0.01$.
      While the marginal $\gamma^\prime$-set for $x_1$ sharply peaks at the center (\textbf{Top} in (b)),
      that for $x_2$ does not (\textbf{Left} in (b)).
      It implies that $x_1$ is more important in the local space.
      \label{appx:theoretical-details:fig:global-vs-local}
    }
    \vspace{-4mm}
  \end{center}
\end{figure*}

\subsection{Why Cannot Prior Works Measure Local Importance?}
\label{appx:theoretical-details:section:difference-from-prior-works}

In this section, we analytically describe why
the prior works~\citeappx{hutter2014efficient,biedenkapp2018cave} cannot local HPI correctly.
Note that we take over the notations from the main paper in this section.

\input{appendices/theoretical-details/analysis-clipped.tex}

\input{appendices/theoretical-details/analysis-changing-hpi.tex}

\input{appendices/theoretical-details/splits.tex}

\input{appendices/theoretical-details/scale-ignorance.tex}

%% file: appendices/theoretical-details/preliminaries.tex
\subsection{Preliminaries}
We first provide the following definitions:
\begin{definition}[$\gamma$-quantile value]
  Given a quantile $\gamma \in (0, 1]$
  and a measurable function $f\mathrm{:} \X \rightarrow \mathbb{R}$,
  $\gamma$-quantile value $f^\gamma \in \mathbb{R}$
  is a real number such that$\mathrm{:}$
  \begin{equation}
    \begin{aligned}
      f^\gamma \coloneqq \inf\biggl\{
      f^\star \in \mathbb{R} ~\biggl|~
      \int_{\xv \in \X} \indic{f(\xv) \leq f^\star}
      \frac{\mu(d\xv)}{\mu(\X)}
      \geq \gamma
      \biggr\}.
    \end{aligned}
  \end{equation}
\end{definition}
\begin{definition}[$\gamma$-set]
  Given a quantile $\gamma \in (0, 1]$,
  $\gamma$-set $\Xg$ is defined as
  $\Xg \coloneqq \{\xv \in \X | f(\xv) \leq f^\gamma\} \in \Bx$.
\end{definition}
\begin{definition}[Marginal $\gamma$-set PDF]
  Given the $\gamma$-set $\Xg$ and its $\gamma$-set $\mathrm{PDF}$ $p(\xv | \Xg)$,
  the marginal $\gamma$-set $\mathrm{PDF}$ for the $d$-th dimension is computed as$\mathrm{:}$
  \begin{equation}
    \begin{aligned}
      p_{d}(x_d | \Xg) \coloneqq \int_{\xv_{-d} \in \X_{-d}}
      p(\xv | \Xg) \mu(d\xv_{-d}).
    \end{aligned}
  \end{equation}
\end{definition}
\begin{definition}[Pearson divergence~\citeappx{pearson1900x}]
  The Pearson divergence between two $\mathrm{PDFs}$
  $p(\xv), q(\xv)$
  is defined as$\mathrm{:}$
  \begin{equation}
    \begin{aligned}
      \pe{p}{q} \coloneqq \int_{\xv \in \X} \biggl(
      \frac{p(\xv)}{q(\xv)} - 1
      \biggr)^2 q(\xv) \mu(d\xv).
    \end{aligned}
  \end{equation}
\end{definition}
The approximation is usually performed via either
Monte-Carlo sampling
or the direct estimation~\citeappx{sugiyama2013direct}.
Figure~\ref{appx:theoretical-details:fig:ped-conceptual}
intuitively presents how we approximate PED.
Throughout this paper, we assume that the support of $q$ includes that of $p$.

%% file: appendices/theoretical-details/global-hpi.tex
\subsection{Global Hyperparameter Importance}
\label{appx:global-hpi:section}
Suppose we would like to analyze a function $f$,
we first introduce the binary function
$b(\xv | \Xg) = \indic{\xv \in \Xg}$,
which is actually a probability measure.
Then the following proposition is derived from this constraint:
\begin{proposition}
  Given the binary function $b(\xv | \Xg)$
  and its $\gamma$-set $\mathrm{PDF}$ $p(\xv | \Xg)$,
  the marginal variance of each dimension $d \in [D]$ is$\mathrm{:}$
  \begin{equation}
    \begin{aligned}
      v_d & = \gamma^2
      \int_{x_d \in \X_d}
      \Biggl(
      \frac{p_d(x_d | \Xg)}{u(\X_d)} - 1
      \Biggr)^2 u(\X_{d}) \mu(dx_d) \\
                    & =
      \gamma^2 ~\pe{p_d(\cdot | \Xg)}{u(\X_d)}.
    \end{aligned}
    \label{appx:theoretical-details:eq:global-hpi}
  \end{equation}
  \label{appx:theoretical-details:proposition:global-hpi}
\end{proposition}
The proof is provided in Appendix~\ref{appx:proofs:section:proof-of-global-hpi}.
It measures PED between the uniform PDF $u(\X_d) = p_d(x_d|\X)$ and the marginal $\gamma$-set PDF $p_d(x_d|\Xg)$.
This HPI measure is generalized to higher orders and we show the formulation in Eq.~(\ref{appx:proofs:eq:general-global-hpi}) of Appendix~\ref{appx:proofs:section:proof-of-global-hpi}.
Note that since a non-uniform sampler yields non-uniform $\D$, it is recommended to use $p_d(x_d | \D^\gamma)$ with $\gamma = 1$ instead of $u(\X_d)$, i.e. $p_d(x_d | \D)$, as discussed in detail in Appendix~\ref{appx:use-cases:section:posthoc-hpo}.
The difference between using $u(\X_d)$ and $p_d(x_d | \D)$ is that while the former shows general HPI in the global space, the latter shows HPI during the search.

%% file: appendices/theoretical-details/local-hpi.tex
\subsection{Local Hyperparameter Importance}
\label{appx:theoretical-details:section:local-hpi}
On top of global HPI, our method can quantify
local HPI using the following theorem
(equivalent to Theorem~\ref{main:methods:theorem:local-importance}):
\begin{theorem}
  Given the binary function $b(\xv|\Xgp)$ and
  the $\gamma^\prime$- and $\gamma$-set
  $\mathrm{PDFs}$ $p(\xv|\Xgp), p(\xv | \Xg)$
  where $\gamma^\prime < \gamma$,
  the local marginal variance of each dimension $d \in [D]$ is$\mathrm{:}$
  \begin{align}
       v^\gamma_d 
       & =
      \biggl(
      \frac{\gamma^\prime}{\gamma}
      \biggr)^2
      \int_{x_d \in \X_d}
      \Biggl(
      \frac{p_d(x_d | \Xgp)}{p_d(x_d | \Xg)}
      - 1
      \Biggr)^2
      p_d(x_d | \Xg)
      \mu(dx_d) \nonumber         \\
       & = \biggl(
      \frac{\gamma^\prime}{\gamma}
      \biggr)^2
      ~\pe{p_d(\cdot | \Xgp)}{p_d(\cdot | \Xg)}.
      \label{appx:theoretical-details:eq:local-hpi}
    \end{align}
  \label{appx:theoretical-details:theorem:local-hpi}
\end{theorem}
This formulation also generalizes to higher orders
as shown in Eq.~(\ref{appendix:proofs:eq:general-local-importance})
of Appendix~\ref{appx:proofs:section:proof-of-local-hpi}.

%% file: appendices/theoretical-details/analysis-clipped.tex
\subsubsection{Case I: Analysis of Clipped Functions}
\label{appx:theoretical-details:section:analysis-clipped}
As mentioned in Section~\ref{main:background:section:local-anova},
\citewithname{hutter2014efficient} proposed global HPI on a clipped function $g(\xv) \coloneqq \min(f(\xv), f^\star)$
as local HPI.
However, since global HPI on the clipped function is still affected
by the search space design,
this is also not a local HPI measure.
To illustrate our implication,
we would like to present a simple example.
Suppose we would like to compute global HPI of
$g(x) \coloneqq \min(x^2, 1)$ defined on $[-L, L]$
as visualized in Figure~\ref{appx:theoretical-details:fig:toy-example-viz}.
Then the global mean and the global variance of $g(x)$ are computed as follows:
\begin{equation}
  \begin{aligned}
    m^{\mathrm{global}}(L)
                           & \coloneqq \frac{1}{2L}\int_{-L}^{L} g(x)dx = 1 - \frac{2}{3L},          \\
    v^{\mathrm{global}}(L) & \coloneqq \frac{1}{2L} \int_{-L}^L (g(x) - m^{\mathrm{global}}(L))^2 dx \\
                           & =  \frac{4}{3L} \biggl(\frac{2}{5} - \frac{1}{3L}\biggr).
  \end{aligned}
  \label{appx:theoretical-details:eq:local-hpi-by-hutter}
\end{equation}
On the other hand, if we limit the integral to the local space
(see the blue-shaded domain in Figure~\ref{appx:theoretical-details:fig:toy-example-viz}),
which is $\{x \in [-L, L]\mid f(x) < 1\} = [-1, 1]$,
the local mean and the local variance are computed as follows:
\begin{equation}
  \begin{aligned}
    m^{\mathrm{local}} & \coloneqq m^{\mathrm{global}}(1) = \frac{1}{3},  \\
    v^{\mathrm{local}} & \coloneqq v^{\mathrm{global}}(1) = \frac{4}{45}.
  \end{aligned}
\end{equation}
Note that our method computes the integral of $\indic{x \in [-1, 1]}$
instead of $g(x)$, but we calculated the integral of $g(x)$ here to
show the difference from the original proposition.
Figure~\ref{appx:theoretical-details:fig:length-vs-hpi}
shows the variation in the global variance
(the local f-ANOVA by \citewithname{hutter2014efficient})
in Eq.~(\ref{appx:theoretical-details:eq:local-hpi-by-hutter})
with respect to the search space design variable $L$.
As can be seen, while the local variance is constant,
the global variance dynamically changes.
It implies that the prior work cannot strictly quantify HPI locally.
The problem is that when we compute the global variance of $g(x)$,
the variance $v$ quickly decays as the domain size $L$ becomes larger.
It means that HPI in some dimensions might be severely underestimated
if we use global HPI on the transformed function such as
$g(x)$.
For this reason, it is important to limit the integral to a local space.

%% file: appendices/theoretical-details/analysis-changing-hpi.tex
\subsubsection{Case II: Analysis of Dynamically Changing HPI}
\label{appx:theoretical-details:section:analysis-on-func}
In this section, we provide a simple example where
local HPI matters, but cannot be quantified by
the prior work~\cite{biedenkapp2018cave}
and analytically show the prior work cannot quantify local HPI
as intended.
We analytically apply Eqs.~(\ref{main:background:eq:global-mean})--(\ref{main:background:eq:marginal-var}) (global HPI)
and Eqs.~(\ref{main:methods:eq:local-mean})--(\ref{main:methods:eq:marginal-local-var}) (our local HPI formula)
to the following function:
\begin{eqnarray}
  f(x_1, x_2) = \left\{
  \begin{array}{ll}
    x_1^2 + x_2^2             & (x_1^2 + x_2^2 \geq 3) \\
    x_1^2 + \frac{x_2^2}{100} & (\mathrm{otherwise})
  \end{array}
  \right.
  \label{appx:theoretical-details:eq:toy-func}
\end{eqnarray}
where $x_1, x_2 \in [-5, 5]$.
The visualization is available in Figure~\ref{appx:theoretical-details:fig:global-vs-local}
and it is clear that while both $x_1, x_2$ is equally important in the global space,
$x_1$ is much more important in the local space $\Xg \coloneqq \{(x_1, x_2) \in [-5,5]\times [-5,5] \mid x_1^2 + x_2^2 \leq 3\}$,
which is the $\gamma \simeq 0.1$-set.

We first derive the global marginal variances of each dimension.
The global marginal means and the global mean of $f(x_1, x_2)$
are as follows:
\begin{eqnarray*}
  f_1(x_1) = \left\{
  \begin{array}{ll}
    x_1^2 + \frac{25}{3}                                    & (|x_1| > \sqrt{3})   \\
    x_1^2 + \frac{25}{3} - \frac{33}{500} (3 - x_1^2)^{3/2} & (\mathrm{otherwise}) \\
  \end{array}
  \right.,
\end{eqnarray*}
\begin{eqnarray*}
  f_2(x_2) = \left\{
  \begin{array}{ll}
    x_2^2 + \frac{25}{3}                                        & (|x_2| > \sqrt{3})   \\
    x_2^2 + \frac{25}{3} - \frac{99}{500}x_2^2 \sqrt{3 - x_2^2} & (\mathrm{otherwise}) \\
  \end{array}
  \right.,
\end{eqnarray*}
\begin{equation}
  \begin{aligned}
    m
     & = \frac{1}{10}\int_{-5}^{5} f_1(x_1) dx_1
    = \frac{1}{10}\int_{-5}^{5} f_2(x_2) dx_2    \\
     & = \frac{50}{3} - \frac{891}{40000}\pi.
  \end{aligned}
\end{equation}
Then we numerically compute the global marginal variances based on the results.
\begin{equation}
  \begin{aligned}
    v_1 & =
    \frac{1}{10}\int_{-5}^5 (f_1(x_1) - m)^2 dx_1 \simeq 56.67,        \\
    v_2 & = \frac{1}{10}\int_{-5}^5 (f_2(x_2) - m)^2 dx_2\simeq 56.53.
  \end{aligned}
\end{equation}
Note that although those values could be analytically computed,
we only show the numerical solutions
to avoid additional complexities.
When using global HPI of PED-ANOVA in Eq.~(\ref{appx:theoretical-details:eq:global-hpi}),
we obtain $2.11$ for both dimensions.
The results show that both dimensions are almost equally important
in the global space as expected.
In fact, the marginal $\gamma$-set PDFs
for each dimension in Figure~\ref{appx:theoretical-details:fig:global-eg-on-toy} coincide
and this confirms the results.

Now, we derive local HPI of $f(x_1, x_2)$ in $\Xg$.
We first calculate the scaling factor:
\begin{equation}
  \begin{aligned}
    V^{\gamma}_1(x_1) & \coloneqq
    \int_{-\sqrt{3 - x_1^2}}^{\sqrt{3 - x_1^2}}
    dx_2 = 2\sqrt{3 - x_1^2}                                                             \\
    V^{\gamma}_2(x_2) & \coloneqq
    \int_{-\sqrt{3 - x_2^2}}^{\sqrt{3 - x_2^2}}
    dx_1 = 2\sqrt{3 - x_2^2}                                                             \\
                      & (\because \indic{\xv \in \Xg} \Rightarrow x_1^2 + x_2^2 \leq 3).
  \end{aligned}
\end{equation}
Then we compute the local marginal means as follows:
\begin{equation}
  \begin{aligned}
    f_1^\gamma(x_1) & =
    \frac{1}{2\sqrt{3-x_1^2}}\int_{-\sqrt{3 - x_1^2}}^{\sqrt{3 - x_1^2}}
    \biggl(x_1^2 + \frac{x_2^2}{100}\biggr) dx_2               \\
                    & = \frac{1}{100} + \frac{299}{300} x_1^2,
  \end{aligned}
\end{equation}
\begin{equation}
  \begin{aligned}
    f_2^\gamma(x_2)
     & =
    \frac{1}{2\sqrt{3-x_2^2}}\int_{-\sqrt{3 - x_2^2}}^{\sqrt{3 - x_2^2}}
    \biggl(x_1^2 + \frac{x_2^2}{100}\biggr) dx_1 \\
     & = 1 - \frac{97}{300} x_2^2,
  \end{aligned}
\end{equation}
Using the local marginal means, the global mean is computed as follows:
\begin{equation}
  \begin{aligned}
    m^\gamma
     & = \int_{-\sqrt{3}}^{\sqrt{3}} f_1^\gamma(x_1)
    \underbrace{\frac{2\sqrt{3-x_1^2}}{3\pi}}_{
    \mathclap{= V^{\gamma}_1(x_1) / Z}
    } dx_1                                           \\
     & = \int_{-\sqrt{3}}^{\sqrt{3}} f_2^\gamma(x_2)
    \underbrace{\frac{2\sqrt{3-x_2^2}}{3\pi}}_{
    \mathclap{= V^{\gamma}_2(x_2) / Z}
    } dx_2
    = \frac{303}{400}.
  \end{aligned}
\end{equation}
Note that $2\sqrt{3 - x_1^2} / 3\pi, 2\sqrt{3 - x_2^2} / 3\pi$
are the marginal $\gamma$-set PDFs and
it refers to $V^{\gamma}_d(x_d)/Z$ of Eq.~(\ref{main:methods:eq:marginal-local-var})
in the main paper.
Then we yield the marginal variances for each dimension as follows:
\begin{align}
  v_1^\gamma & =
  \int_{-\sqrt{3}}^{\sqrt{3}} (f_1^\gamma(x_1) - m^\gamma)^2
  \frac{2\sqrt{3 - x_1^2}}{3\pi} dx_1 \nonumber               \\
             & = \frac{89401}{160000} \simeq 0.559, \nonumber \\
  v_2^\gamma & =
  \int_{-\sqrt{3}}^{\sqrt{3}} (f_2^\gamma(x_2) - m^\gamma)^2
  \frac{2\sqrt{3 - x_2^2}}{3\pi} dx_2               \nonumber \\
             & = \frac{9409}{160000}
  \simeq 0.0588.
\end{align}
Additionally, local HPI of PED-ANOVA in Eq.~(\ref{appx:theoretical-details:eq:local-hpi})
yields $v_1^\gamma \simeq 8.97$ and $v_1^\gamma \simeq 0.181$.
This shows that
$x_1$ is more important in the local space
and this interpretation can be seen in Figure~\ref{appx:theoretical-details:fig:local-eg-on-toy} as well.
In fact, while the marginal $\gamma$-set PDF
for $x_2$ is close to uniform,
that for $x_1$ sharply peaks at the center.
This indicates that $x_1$ is indeed more important in the local space.

In contrast to our method,
the local f-ANOVA based on Eq.~(\ref{main:background:eq:local-hpi-biedenkapp})
obtains the following results:
\begin{equation}
  \begin{aligned}
    m_1 & = \frac{1}{10}\int_{-5}^{5} f(x_1, 0)dx_1 = \frac{25}{3},                          \\
    m_2 & = \frac{1}{10}\int_{-5}^{5} f(0, x_2)dx_2 = \frac{25}{3} - \frac{99\sqrt{3}}{500}, \\
    v_1 & = \frac{1}{10}\int_{-5}^{5} (f(x_1, 0) - m_1)^2 dx_1 = \frac{500}{9} \simeq 55.6,  \\
    v_2 & = \frac{1}{10}\int_{-5}^{5} (f(0, x_2) - m_2)^2 dx_2 \simeq 60.5,
  \end{aligned}
\end{equation}
where we plugged in
the optimal solution $\xopt = [0, 0]^\top$
for $f(x_1, x_2)$
to Eq.~(\ref{main:background:eq:local-hpi-biedenkapp}).
The conclusion drawn from the results is that both dimensions are almost equally important in the local space;
however, it clearly contradicts
the intuition drawn from Figure~\ref{appx:theoretical-details:fig:local-eg-on-toy}.

%% file: appendices/theoretical-details/splits.tex
\subsection{Riemann Split vs Lebesgue Split}
\label{appx:theoretical-details:section:splits}
As discussed in Section~\ref{main:methods:section:local-hpi}, the Lebesgue split reduces the number of control parameters and allows us to focus on the analysis in promising domains.
Furthermore, it provides more reliable quantification because the subspaces obtained by the Lebesgue split guarantee to have observations while those by the Riemann split do not.
This also leads to the de-bias in the quantification incurred by sampling bias~\citeappx{moosbauer2021explaining}.
In principle, the sampling bias is caused when observations are not generated from a uniform sampler, e.g. when using Bayesian optimization.
It may lead to incorrect quantification of HPI because the integral of f-ANOVA in Eqs.~(\ref{main:background:eq:global-mean})--(\ref{main:background:eq:marginal-var}) taken over the uniform PDF.
However, as the local HPI in Eqs.~(\ref{main:methods:eq:local-mean})--(\ref{main:methods:eq:marginal-local-var}) takes the integral over the distribution built by the observations, the sampling bias can be removed.
Note that the difference between using $u(\X_d)$ and $p_d(x_d | \D)$ is that while the former shows general HPI in the global space, the latter shows HPI during the search and we discuss more details in Appendix~\ref{appx:use-cases:section:posthoc-hpo}. 
The sampling bias is a big problem in the Riemann split because it does not guarantee to have observations in the specified subspace.
It implies that when we do not have no observations near or inside the specified subspace, the surrogate model could be very unreliable and we could even get zero variances for all the parameters, which means that all the parameters are trivial.
The drawback of the Lebesgue split is to not be able to provide an easy-to-understand subspace compared to the Riemann split.
Especially when practitioners have a specific subspace in their mind, the Riemann split could be a better choice.

%% file: appendices/theoretical-details/scale-ignorance.tex
\begin{figure*}[t]
  \centering
  \includegraphics[width=0.9\textwidth]{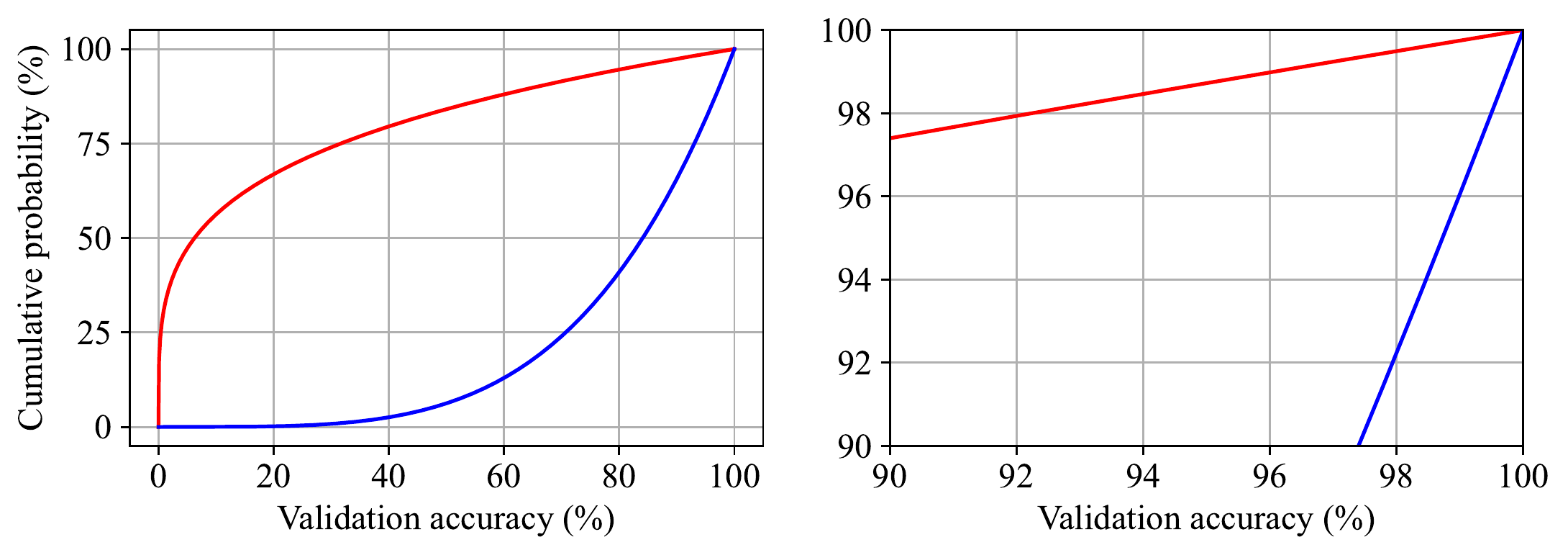}
  \vspace{-3mm}
  \caption{
    The distributions of two different functions.
    The horizontal axis shows the function value of $f(\xv)$
    (we consider the validation accuracy of image classification)
    and the vertical axis shows the cumulative probability
    of a certain validation accuracy value.
    For example,
    the validation accuracy of $60\%$
    in the blue line of the left figure is
    about the bottom $15\%$
    and that of $80\%$
    the blue line of the left figure is about the bottom $40\%$.
    \textbf{Left}: the whole range of distributions.
    The red line shows that the bottom-$25\%$ configurations
    still exhibits the validation accuracy of about $0\%$
    and
    the blue line shows that the bottom-$25\%$ configurations
    already exhibits the validation accuracy of about $70\%$.
    \textbf{Right}:
    the magnified figure of the left figure.
    While the blue line can improve only about $3\%$
    from the top-$10\%$ to the top (i.e. $\Delta f_1 \simeq 0.03$),
    the blue line can improve $97\%$ from the bottom to the top-$10\%$ (i.e. $\Delta f_2 \simeq 0.97$).
    For the red line, $\Delta f_1 \simeq 0.65$
    and $\Delta f_2 = 0.35$.
  }
  \vspace{-3mm}
  \label{appx:theoretical-details:fig:scale-ignorance}
\end{figure*}

\subsection{Benefits of Scale Ignorance}
\label{appx:theoretical-details:section:scale-ignorance}
We first note that a function $f(\xv)$ is higher-is-better in this section,
but a function $f(\xv)$ is lower-is-better in the other sections.
As discussed in Section~\ref{main:validation:section:scale-ignorance},
the benefit of scale ignorance is to be able to ignore the contribution
from the tail.
We prepared two functions with different function value distributions
in Figure~\ref{appx:theoretical-details:fig:scale-ignorance}.
More formally, the distributions show:
\begin{equation}
\begin{aligned}
  \mathbb{P}(f^\star) \coloneqq \frac{1}{\mu(\X)}\int_{\xv \in \X} \indic{f(\xv) \leq f^\star} \mu(d\xv).
\end{aligned}
\end{equation}
While the red line shows a quick evolution of the validation accuracy
at the high cumulative probability domain,
the blue line shows a slow evolution.
In principle, when a function exhibits
a slow evolution at the high cumulative probability domain
as in the blue line,
the function exhibits a quick evolution
at the low cumulative probability domain,
and thus the variance is biased toward the variation
in the low cumulative probability domain.
For example, we are intrinsically interested in
what makes difference in the last few percent
of the validation accuracy,
but not the first trivial $80 \sim 90\%$ variation.
More formally, we need to pay attention to:
\begin{enumerate}
  \item \textbf{Variation from bottom}: $\Delta f_1 \coloneqq \mathbb{P}^{-1}(1 - \gamma) - \mathbb{P}^{-1}(0)$,
  \item \textbf{Variation from top}: $\Delta f_2 \coloneqq \mathbb{P}^{-1}(1) - \mathbb{P}^{-1}(1 - \gamma)$.
\end{enumerate}
When $\Delta f_1 \gg \Delta f_2$ as in the blue line,
the original f-ANOVA is likely to yield a global variance similar
to the scale of $\Delta f_1^2$ rather than $\Delta f_2^2$
although we cannot mathematically guarantee it.
In Section~\ref{main:validation:section:scale-ignorance}, we called \textbf{variation from bottom} ``tail''.
As discussed in Appendix~\ref{appx:experiments:section:dataset-detail},
JAHS-Bench-201 has similar distributions as the blue line in Figure~\ref{appx:theoretical-details:fig:scale-ignorance}.
PED-ANOVA will not be affected by \emph{variation from bottom}
due to the scale ignorance nature
and it captures HPI for better performance in the last effort more clearly.

%% file: appendices/proofs/main.tex
\section{Proofs}
\label{appx:proofs:section}

\input{appendices/proofs/assumptions.tex}

\input{appendices/proofs/global-importance.tex}

\input{appendices/proofs/local-importance.tex}

\input{appendices/proofs/discretization-error.tex}

%% file: appendices/proofs/assumptions.tex
\subsection{Assumptions}
\label{appx:proofs:section:assumptions}
We assume the following:
\begin{enumerate}
  \item Objective function $f: \X \rightarrow \mathbb{R}$
  is a square-integrable measurable function defined over the compact convex
  measurable subset $\X \subseteq \mathbb{R}^D$, and
  \item The $\gamma$-set PDF always exists and
  $p(\xv|\Xg) \propto b(\xv|\Xg) \coloneqq b^\gamma(\xv) = \indic{\xv \in \Xg}$
  holds.
\end{enumerate}
Strictly speaking, we cannot guarantee that the $\gamma$-set PDF
always exists;
however, we formally assume that the $\gamma$-set PDF
exists by considering 
the empirical distribution and
the formal derivative of the step function
as the Dirac delta function.
As defined above, we use $b^\gamma(\xv)$ for simplicity in this section.

%% file: appendices/proofs/global-importance.tex
\subsection{Proof of Proposition \ref{appx:theoretical-details:proposition:global-hpi}}
\label{appx:proofs:section:proof-of-global-hpi}
We first compute the marginal
mean of the binary function $b^\gamma(\xv)$:
\begin{equation}
  \begin{aligned}
     f_d(x_d | \Pcal_d) \coloneqq 
    \frac{V^{\gamma}_d(x_d)}{\mu(\X_{-d})}
    &= \int_{\xv_{-d} \in \X_{-d}} b^\gamma(\xv | x_d) \frac{\mu(d\xv_{-d})}{\mu(\X_{-d})}
  \end{aligned}
  \label{appx:proofs:eq:def-of-local-volume}
\end{equation}
where $\xv_{-d} \in \mathbb{R}^{D - 1}$ is $\xv$
without the $d$-th dimension.
We also note that the following holds by definition of the binary function:
\begin{equation}
  \begin{aligned}
    m = \gamma = \int_{\xv \in \X} b^\gamma(\xv) \frac{\mu(d\xv)}{\mu(\X)}.
  \end{aligned}
  \label{appx:proofs:eq:binary-expectation-is-gamma}
\end{equation}
Notice that $m$ is the second term of RHS in Eq.~(\ref{appx:background:eq:general-zero-centered-marginal-mean})
because $\Pcal_{-d} = \{\emptyset\}$ holds.
Then we first prove the following lemma:
\begin{lemma}
  The following holds under the assumptions of this section$\mathrm{:}$
  \begin{equation}
    \begin{aligned}
      \gamma = \int_{x_d \in \X_d} \frac{V^{\gamma}_d(x_d)}{\mu(\X_{-d})}
      \frac{\mu(dx_d)}{\mu(\X_d)} 
      = \int_{x_d \in \X_d} \frac{V^{\gamma}_d(x_d)}{\mu(\X)} \mu(dx_d).
    \end{aligned}
  \end{equation}
  \label{appx:proofs:lemma:integral-of-marginal-dist}
\end{lemma}
\begin{proof}
  By definition, the following equality holds$\mathrm{:}$
  \begin{equation}
    \begin{aligned}
      \int_{x_d \in \X_d} & \frac{V^{\gamma}_d(x_d)}{\mu(\X_{-d})}
      \frac{\mu(dx_d)}{\mu(\X_d)} = \\                                                      
                           &\int_{x_d \in \X_d} \int_{\xv_{-d} \in \X_{-d}} b^\gamma(\xv | x_d)
      \frac{\mu(d\xv_{-d})}{\mu(\X_{-d})}
      \frac{\mu(dx_d)}{\mu(\X_d)}.
    \end{aligned}
  \end{equation}
  Since the Lebesgue measure is a product measure
  and $b^\gamma(\xv) \geq 0$,
  the Fubini's theorem holds, and thus we obtain the following$\mathrm{:}$
  \begin{equation}
    \begin{aligned}
      \mathrm{LHS} = \int_{\xv \in \X} b^\gamma(\xv) \frac{\mu(d\xv)}{\mu(\X)}
      = \gamma~(\because \mathrm{Eq.~}(\ref{appx:proofs:eq:binary-expectation-is-gamma})).
    \end{aligned}
  \end{equation}
  This completes the proof.
\end{proof}
Using this lemma, we prove Proposition~\ref{appx:theoretical-details:proposition:global-hpi}.
\begin{proof}
  From the assumption, $b^\gamma(\xv)$
  is proportional to the $\gamma$-set $\mathrm{PDF}$ $p(\xv | \Xg)$,
  the marginalization of $b^\gamma(\xv)$ is also proportional to that of the $\gamma$-set $\mathrm{PDF}$.
  From $\mathrm{Lemma~\ref{appx:proofs:lemma:integral-of-marginal-dist}}$,
  the scale to equalize this marginalization is $\gamma\mu(\X)$;
  therefore, the marginal $\gamma$-set $\mathrm{PDF}$
  is computed as $p_d(x_d | \Xg) = V^{\gamma}_d(x_d) / (\gamma \mu(\X))$ and
  we obtain the following marginal variance$\mathrm{:}$
  \begin{equation}
    \begin{aligned}
      v_d & =
      \int_{x_d \in \X_d} \biggl(
        \underbrace{\frac{V^{\gamma}_d(x_d)}{\mu(\X_{-d})}}_{
          f_d(x_d | \Pcal_d)
        } - \underbrace{\gamma}_{
          m
        }
        \biggr)^2 \frac{\mu(dx_d)}{\mu(\X_d)} 
        ~(\because \mathrm{Eq.}~(\ref{main:background:eq:marginal-var})) \\
                  & = \mathbb{E}\biggl[
                    \biggl(\frac{V^{\gamma}_d(x_d)}{\mu(\X_{-d})} - \gamma
                    \biggr)^2\biggr]                                                       \\
                  & = \mathbb{E}[(\gamma \mu(\X_d)p_d(x_d | \Xg) - \gamma)^2]                                       \\
                  & = \gamma^2 \mathbb{E}\Biggl[
        \Biggl(\frac{p_d(x_d | \Xg)}{u(\X_d)} - 1
        \Biggr)^2
        \Biggr]
    \end{aligned}
  \end{equation}
  where we used $u(\X_d) = 1 / \mu(\X_d)$
  and the expectation is taken with respect to
  $u(\X_d)$.
  This completes the proof.
\end{proof}
Using Eqs.~(\ref{appx:background:eq:general-marginal-mean})--(\ref{appx:background:eq:general-marginal-var})
and considering the same procedure as in 
Proposition~\ref{appx:theoretical-details:proposition:global-hpi},
higher orders of HPI for an arbitrary
combination of dimensions 
$s \subseteq [D]$
can be computed as:
\begin{equation}
\begin{aligned}
  v_s
  &= \mathbb{E}\biggl[
    \biggl(
      \frac{V^{\gamma}_s(\xv_s)}{\mu(\X_{-s})}
       -
      \sum_{s^\prime \in \Pcal_{-s}}
      \frac{V^{\gamma}_{s^\prime}(\xv_{s^\prime})}{\mu(\X_{-s^\prime})}
    \biggr)^2
  \biggr] \\
  &= \mathbb{E}\biggl[
    \biggl(
      \frac{\gamma p_s(\xv_s | \Xg)}{u(\X_s)} -
      \sum_{s^\prime \in \Pcal_{-s}}
      \frac{\gamma p_{s^\prime}(\xv_{s^\prime} | \Xg)}{u(\X_{s^\prime})}
    \biggr)^2
  \biggr] \\
  &~~~~~~\biggl(
    \because p_s(\xv_s | \Xg) = \frac{u(\X) V^{\gamma}_s(\xv_s)}{\gamma}
  \biggr)\\
  &= \gamma^2 \mathbb{E}\biggl[
    \biggl(
      \frac{p_s(\xv_s | \Xg)}{u(\X_s)} -
      \sum_{s^\prime \in \Pcal_{-s}}
      \frac{p_{s^\prime}(\xv_{s^\prime} | \Xg)}{u(\X_{s^\prime})}
    \biggr)^2
  \biggr]
\end{aligned}
\label{appx:proofs:eq:general-global-hpi}
\end{equation} 
where the expectation is taken with respect to $u(\X_s)$
and we defined $p_{\emptyset} = 1$.
Note that $\Pcal_{-s}$ includes an empty set $\emptyset$.
For example, when $s = [2]$,
the second term becomes:
\begin{equation}
\begin{aligned}
  \sum_{s^\prime \in \{\Scal_1, \Scal_2, \Scal_{\emptyset}\}}
  \frac{p_{s^\prime}(\xv_{s^\prime} | \Xg)}{u(\X_{s^\prime})}
  = \frac{p_1(x_1 | \Xg)}{u(\X_1)}
  + \frac{p_2(x_2 | \Xg)}{u(\X_2)}
  + 1.
\end{aligned}
\end{equation}
Eq.~(\ref{appx:proofs:eq:general-global-hpi}) falls back to Eq.~(\ref{appx:theoretical-details:eq:global-hpi}) when $s = \Scal_d$.

%% file: appendices/proofs/local-importance.tex
\subsection{Proof of Theorem \ref{main:methods:theorem:local-importance}}
\label{appx:proofs:section:proof-of-local-hpi}
\begin{proof}
  Based on $\mathrm{Eq.~(\ref{main:methods:eq:marginal-local-mean})}$,
  we first compute the marginal mean of the binary function $b^{\gamma^\prime}(\xv)$
  using $V^{\gamma}_d(x_d) = \gamma \mu(\X) p_d(x_d|\Xg)\mathrm{:}$
  \begin{equation}
    \begin{aligned}
      f_d^\gamma(x_d | \Pcal_d)
       & = \int_{\xv_{-d} \in \X_{-d}}
      b^{\gamma^\prime}(\xv | x_d) \frac{b^\gamma(\xv|x_d)\mu(d\xv_{-d})}{
        V^{\gamma}_d(x_d)
      }                                                                          \\
       & = \int_{\xv_{-d} \in \X_{-d}}
      b^{\gamma^\prime}(\xv | x_d) \frac{\mu(d\xv_{-d})}{
        V^{\gamma}_d(x_d)
      }\\
       &(\because \Xgp \subseteq \Xg \Rightarrow b^\gamma(\xv)b^{\gamma^\prime}(\xv) = b^{\gamma^\prime}(\xv))                                           \\
       & = \frac{V^{\gamma^\prime}_d(x_d)}{V^{\gamma}_d(x_d)}
      = \frac{\gamma^\prime \mu(\X) p_d(x_d|\Xgp) }{\gamma \mu(\X) p_d(x_d|\Xg)} \\
       & = \frac{\gamma^\prime}{\gamma}\frac{p_d(x_d|\Xgp)}{p_d(x_d|\Xg)}.
    \end{aligned}
    \label{appx:proofs:eq:marginal-mean-is-density-ratio}
  \end{equation}
  where we define $f_d^\gamma(x_d | \Pcal_d) = 1$
  if $V^{\gamma}_d(x_d) = 0$.
  By definition of the binary function, the following holds$\mathrm{:}$
  \begin{equation}
    \begin{aligned}
      m^\gamma = \frac{\gamma^\prime}{\gamma} =
      \int_{\xv \in \Xg} b^{\gamma^\prime}(\xv) \frac{\mu(d\xv)}{\mu(\Xg)}.
    \end{aligned}
    \label{appx:proofs:eq:binary-expectation-is-gamma-ratio}
  \end{equation}
  Using $\mathrm{Eqs.~(\ref{main:methods:eq:marginal-local-var}),(\ref{appx:proofs:eq:marginal-mean-is-density-ratio}),(\ref{appx:proofs:eq:binary-expectation-is-gamma-ratio})}$,
  the marginal variance is computed as follows:
  \begin{equation}
    \begin{aligned}
      v^\gamma_d 
      & =
      \int_{x_d \in \X_d} \biggl(
      \underbrace{\frac{V^{\gamma^\prime}_d(x_d)}{V^{\gamma}_d(x_d)}}_{
          f_d^\gamma(x_d | \Pcal_d)
        }
      -
      \underbrace{\frac{\gamma^\prime}{\gamma}}_{
          m^\gamma
        }
      \biggr)^2 \\
      &~~~~~~~~~~~~~~~~~~~~~~~~ \times \underbrace{
        \frac{
          \int_{\xv_{-d} \in \X_{-d}} b^\gamma(\xv|x_d)\mu(d\xv_{-d})
        }{
          \int_{\xv \in \X} b^\gamma(\xv|x_d)\mu(d\xv)
        }
      }_{p_d(x_d|\Xg) = \frac{V^{\gamma}_d(x_d)}{\gamma \mu(\X)}}
      \mu(dx_d)                                                     \\
                   & = \biggl(\frac{\gamma^\prime}{\gamma}\biggr)^2
      \int_{x_d \in \X_d} \biggl(
      \frac{p_d(x_d|\Xgp)}{p_d(x_d|\Xg)}
      - 1\biggr)^2
      p_d(x_d|\Xg)\mu(dx_d)                                         \\
                   & = \biggl(
      \frac{\gamma^\prime}{\gamma}
      \biggr)^2
      \mathbb{E}\biggl[
        \biggl(
        \frac{p_d(x_d | \Xgp)}{p_d(x_d | \Xg)} - 1
        \biggr)^2
        \biggr]
    \end{aligned}
  \end{equation}
  where the expectation is taken over
  $p_d(x_d|\Xg)$.
  This completes the proof.
\end{proof}

Using Eqs.~(\ref{main:background:eq:global-mean})--(\ref{main:background:eq:marginal-var})
and considering the same procedure as in
Theorem~\ref{main:methods:theorem:local-importance},
higher orders of HPI for an arbitrary
combination of dimensions
$s \subseteq [D]$
can be computed as:
\begin{equation}
  \begin{aligned}
    v^\gamma_s
    & = \mathbb{E}\biggl[
      \biggl(
      \frac{V^{\gamma^\prime}_s(\xv_s)}{V^{\gamma}_s(\xv_s)}
      -
      \sum_{s^\prime \in \Pcal_{-s}}
      \frac{V^{\gamma^\prime}_{s^\prime}(\xv_{s^\prime})}{V^{\gamma}_{s^\prime}(\xv_{s^\prime})}
      \biggr)^2
      \biggr]                                                       \\
               & = \mathbb{E}\biggl[
      \biggl(
      \frac{\gamma^\prime p_s(\xv_s | \Xgp)}{\gamma p_s(\xv_s | \Xg)} -
      \sum_{s^\prime \in \Pcal_{-s}}
      \frac{\gamma^\prime p_{s^\prime}(\xv_{s^\prime} | \Xgp)}{\gamma p_{s^\prime}(\xv_{s^\prime} | \Xg)}
      \biggr)^2
      \biggr]                                                       \\
               & = \biggl(\frac{\gamma^\prime}{\gamma}\biggr)^2
    \mathbb{E}\biggl[
      \biggl(
      \frac{p_s(\xv_s | \Xgp)}{p_s(\xv_s | \Xg)} -
      \sum_{s^\prime \in \Pcal_{-s}}
      \frac{ p_{s^\prime}(\xv_{s^\prime} | \Xgp)}{ p_{s^\prime}(\xv_{s^\prime} | \Xg)}
      \biggr)^2
      \biggr]
  \end{aligned}
  \label{appendix:proofs:eq:general-local-importance}
\end{equation}
where the expectation is taken with respect to $p_s(\xv_s | \Xg)$.
Recall that we defined $p_{\emptyset} = 1$ and $\Pcal_{-s}$ includes an empty set $\emptyset$.
For example, when $s = [2]$,
the second term becomes:
\begin{equation}
\begin{aligned}
  \sum_{s^\prime \in \{\Scal_1, \Scal_2, \Scal_{\emptyset}\}}
  \frac{p_{s^\prime}(\xv_{s^\prime} | \Xgp)}{p_{s^\prime}(\xv_{s^\prime} | \Xg)}
  = \frac{p_1(x_1 | \Xgp)}{p_1(x_1 | \Xg)}
  + \frac{p_2(x_2 | \Xgp)}{p_2(x_2 | \Xg)}
  + 1.
\end{aligned}
\end{equation}
Eq.~(\ref{appendix:proofs:eq:general-local-importance}) falls back to Eq.~(\ref{appx:theoretical-details:eq:local-hpi}) when $s = \Scal_d$.

%% file: appendices/proofs/discretization-error.tex
\subsection{Proof about Discretization Error}
\label{appx:proofs:section:discretization-error}
In this section, we prove the maximum discretization error
of PED and the statement is the following proposition:
\begin{proposition}
  Suppose
  the function $h(x_d) \coloneqq (p_d(x_d|\Xgp) / p_d(x_d|\Xg) - 1)^2$
  is integrable and Lipschitz continuous with a Lipschitz constant of $C \in \mathbb{R}_{\geq 0}$
  and we discretize the domain of this function in $n_d$ grids at the even interval,
  then the discretization error of the local marginal variance is bounded by $O(\frac{C}{n_d})$.
  \label{appx:proofs:proposition:discretization-error}
\end{proposition}
\begin{proof}
  We define the domain $\X_d$ as $[L, R] (L, R \in \mathbb{R}, L < R)$ and the grid points as $\{g_n\}_{n=0}^{n_d - 1} \coloneqq \{L + ns\}_{n=0}^{n_d - 1}$ where $s \coloneqq (R - L) / (n_d - 1)$ is a step size.
  Furthermore, we use the notation $p_d^\gamma(x_d) \coloneqq p_d(x_d | \X^\gamma)$.
  Since the discretization error is maximized when
  $h(x_d)$ is monotonically increasing or decreasing
  with the maximum possible slope of $C$,
  the maximum possible discretization error $\epsilon_{\max}$
  is computed as follows$\mathrm{:}$
  \begin{equation}
    \begin{aligned}
       & \frac{\epsilon_{\max}}{\gamma^{\prime ~2}}      \\
       & \coloneqq
      \sum_{n=0}^{n_d - 2}
      \int_{g_n}^{g_{n+1}}
      p_d^\gamma(x)\bigl(
      \bigl(
      h(x) + C(x - g_n)
      \bigr) - h(x)\bigr) dx                             \\
       & =
      \sum_{n=0}^{n_d - 2}
      \int_{g_n}^{g_{n+1}}
      p_d^\gamma(x) C (x - g_n) dx                       \\
       & \leq
      \sum_{n=0}^{n_d - 2}
      \int_{g_n}^{g_{n+1}} sC p_d^\gamma(x) dx
      = sC \sum_{n=0}^{n_d - 2}
      \int_{g_n}^{g_{n+1}} p_d^\gamma(x) dx              \\
       & = sC \int_{L}^{R} p_d^\gamma(x) dx = sC~\biggl(
      \because \int_{L}^{R}p_d^\gamma(x)dx = 1
      \biggr)                                            \\
       & = \frac{C(R - L)}{n_d - 1}.                     \\
    \end{aligned}
  \end{equation}
  From the fact that the maximum discretization error
  $\epsilon_{\max} = \frac{\gamma^{\prime 2} C (R - L)}{n_d - 1}$,
  it is obvious that the discretization error
  is bounded by $O(\frac{C}{n_d})$
  and this completes the proof.
\end{proof}
Note that although our result has a term of the domain size $R - L$, the Lipschitz constant $C$ is inversely proportional to $R - L$ for $h$, and thus the overall order does not change largely.

%% file: appendices/experiments/main.tex
\begin{table*}
  \begin{center}
    \caption{
      The search space of JAHS-Bench-201
      and the grid space used in our experiments.
      \ta~\protect\citeappx{muller2021trivialaugment}
      is a data augmentation method
      and Edge 1 -- 6 follows the NAS-Bench-201~\protect\citeappx{dong2020bench}
      search space.
      This grid search space has
      $7 \times 7 \times 3 \times 3 \times 3 \times 2 \times 5^6 \simeq~$41M
      configurations in total.
      \textbf{Hyperparameter}: the names of each HP.
      \textbf{Parameter type}: the type of each HP.
      Type is either \texttt{continuous}, \texttt{discrete}, or \texttt{categorical}. 
      \textbf{Choices used in our experiments}: the values
      of each grid used in our experiments.
      While categorical and discrete parameters take exactly the same grids as in the original benchmark dataset,
      continuous take 7 grids for each HP to make
      the dataset size finite.
    }
    \label{appx:experiments:tab:jahs-search-space}
    \begin{tabular}{lll}
      \toprule
      Hyperparameter                                     & Parameter type & Choices used in our experiments                                                                                  \\
      \midrule
      Learning rate                                      & Continuous     & \{1e-3, 3e-3, 1e-2, 3e-2, 1e-1, 3e-1, 1e-0\}                                              \\
      Weight decay                                       & Continuous     & \{1e-5, 3e-5, 1e-4, 3e-4, 1e-3, 3e-3, 1e-2\}                                              \\
      Activation function                                & Categorical    & \{\texttt{relu}, \texttt{hardswish}, \texttt{mish}\}                                                                 \\
      TrivialAugment & Categorical    & \{\texttt{True}, \texttt{False}\}                                                                           \\
      \midrule
      Depth multiplier                                   & Discrete       & \{1, 3, 5\}                                                                               \\
      Width multiplier                                   & Discrete       & \{4, 8, 16\}                                                                              \\
      \midrule
      Operation 1 -- 6                                        & Categorical    & \{\texttt{skip-connection}, \texttt{none}, \texttt{bn-conv3x3}, \texttt{bn-conv1x1}, \texttt{avgpool3x3}\} \\
      \bottomrule
    \end{tabular}
  \end{center}
\end{table*}

\begin{figure*}[t]
  \centering
  \includegraphics[width=0.98\textwidth]{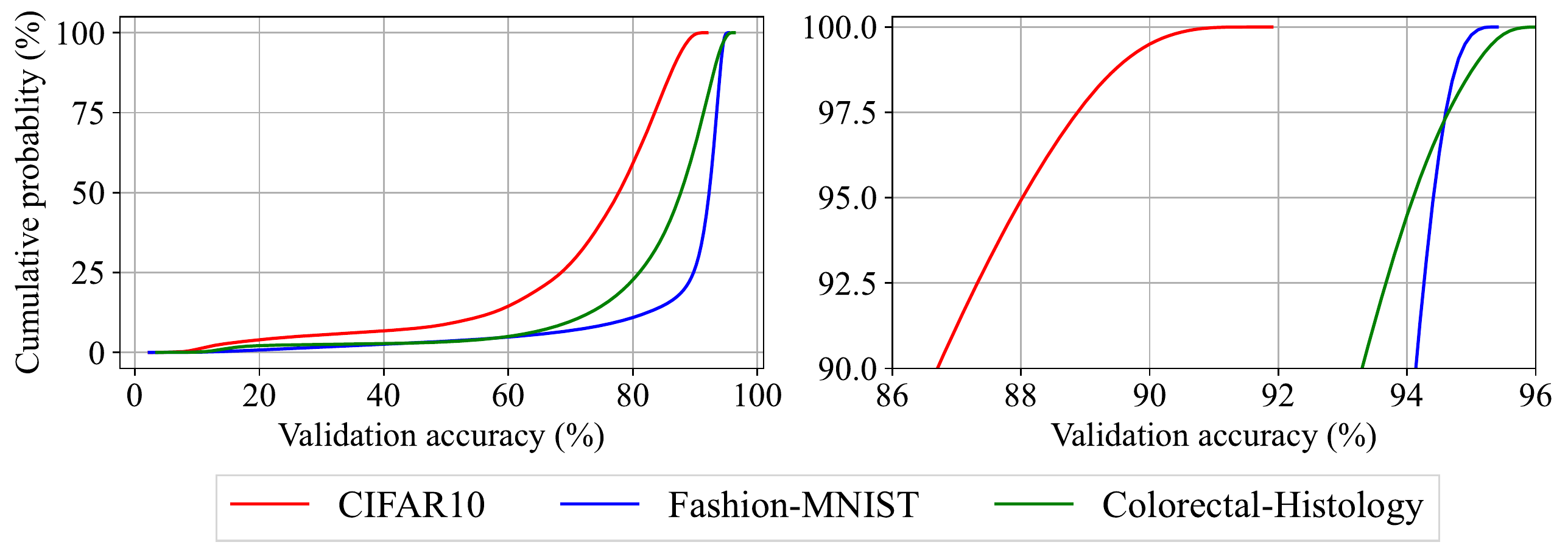}
  \caption{
    The cumulative distributions of the validation accuracy in each dataset
    of JAHS-Bench-201.
    As discussed in Appendix~\ref{appx:theoretical-details:section:scale-ignorance}, the distributions exhibit $\Delta f_1 \ll \Delta f_2$.
    \textbf{Left}: 
    the whole range of the distributions.
    The validation accuracy slowly increases in the high cumulative probability domain,
    and thus the variation in the high cumulative probability domain is relatively less important
    as in the blue line of Figure~\ref{appx:theoretical-details:fig:scale-ignorance}.
    \textbf{Right}: the magnified figure of the left figure.
    This figure tells us that each dataset still has a room for improvement from the top $10\%$.
  }
  \label{appx:experiments:fig:datasets-performance-dist}
\end{figure*}

\begin{figure}
  \centering
  \includegraphics[width=0.49\textwidth]{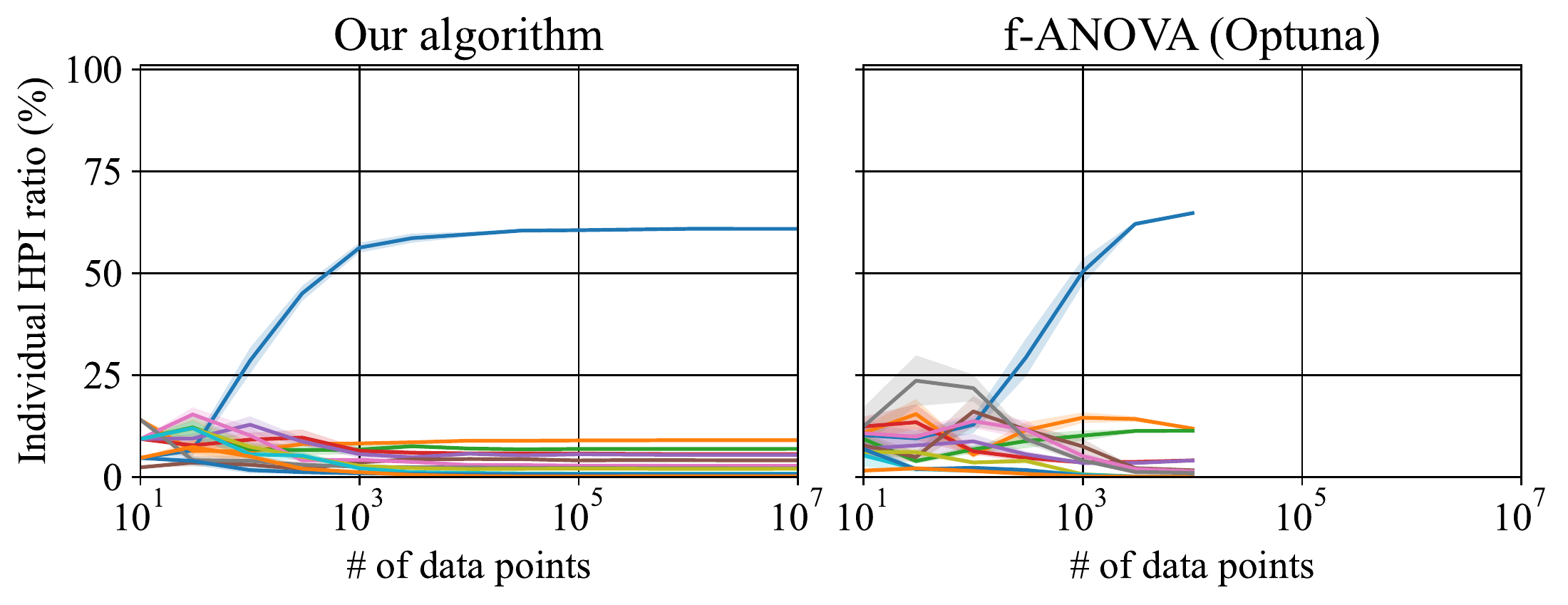}
  \caption{
    The plot for the sample efficiency of our algorithm
    using \texttt{CIFAR10} in JAHS-Bench-201.
    The horizontal axis is the number of data points to be used
    in the training of our algorithm (PED-ANOVA) and f-ANOVA,
    and the vertical axis is the HPI ratio of each HP.
    The colors of plots were determined by the ranking of HPI
    in each setting
    and the weak color band shows the standard error of
    each HPI ratio over 10 independent runs.
    \textbf{Left}:
    the results for our algorithm.
    The HPI ratio starts to converge from $10^4$ data points.
    \textbf{Right}:
    the results for f-ANOVA by Optuna.
    Due to the expensive computation,
    we could run only until $10^4$ data points
    and it did not exhibit convergence.
  }
  \label{appx:experiments:fig:sample-efficiency}
\end{figure}

\section{Additional Experiments for Real-World Usecase by JAHS-Bench-201}
\label{appx:experiments:section}

\input{appendices/experiments/datasets.tex}

\input{appendices/experiments/sample-efficiency.tex}

\input{appendices/experiments/additional-results.tex}

%% file: appendices/experiments/datasets.tex
\subsection{Details of JAHS-Bench-201}
\label{appx:experiments:section:dataset-detail}
JAHS-Bench-201~\citeappx{bansal2022jahs} is a collection of surrogate benchmarks for HPO which has one of the largest search spaces in extant literature.
In a surrogate benchmark, we provide HP configurations to a surrogate model,
e.g. Random Forest, or XGBoost,
and then the surrogate model returns the predicted performance metric values
for the corresponding HP configurations.
Those surrogates are trained on a set of observations, which are pairs of HP configurations
and the performance metric.
This is in contrast to tabular benchmarks,
which query pre-recorded performance metric values
from a static table and, thus,
cannot handle continuous parameters,
whereas surrogate benchmarks can.
In our experiments, we used the grid space in Table~\ref{appx:experiments:tab:jahs-search-space}
with 12 dimensions, resulting in 41M HP configurations.
Note that we fixed the fidelity parameters
``Resolution Multiplier'' and ``Training Epochs''
in the original paper~\cite{bansal2022jahs} to 1.0 and 200, respectively.
Here, ``Choices used in our experiments'' shows the lattice points ($N = 41{,}343{,}750$) used in our experiments,
which includes our discretization of the continuous parameters \texttt{Learning Rate} and \texttt{Weight Decay} from JAHS-Bench-201.
This benchmark provides XGBoost surrogate models
that predict the validation accuracy (and many other metrics)
of neural networks on
three different image classification datasets
(\texttt{CIFAR10}, \texttt{Fashion-MNIST}, and \texttt{Colorectal-Histology})
with each HP configuration. 
To train the surrogate models,
the authors trained deep neural networks
with 161M different HP configurations.
In Figure~\ref{appx:experiments:fig:datasets-performance-dist},
we show the distributions of the validation accuracy on each dataset
and those distributions show a slow evolution at
the high cumulative probability domain as the blue line in Figure~\ref{appx:theoretical-details:fig:scale-ignorance}.

%% file: appendices/experiments/sample-efficiency.tex
\subsection{Sample Efficiency}
\label{appx:experiments:section:sample-efficiency}
In this experiment, we would like to show
that it is important to use as many data points as possible
to obtain precise interpretation.
Figure~\ref{appx:experiments:fig:sample-efficiency}
presents the HPI ratios by each algorithm and
their standard error over $10$ independent runs.
Based on the result of our algorithm,
the HPI ratios start to converge from $10^4$ training data points.
On the other hand, Figure~\ref{appx:experiments:fig:sample-efficiency}
shows that f-ANOVA does not converge with $10^4$ data points.
It implies that we should use as many data points as possible
to analyze benchmark datasets,
and thus the scalability of our algorithm is desirable.

%% file: appendices/experiments/additional-results.tex
\begin{table*}[t]
  \begin{center}
    \caption{
      HPI of \texttt{Fashion-MNIST} and \texttt{Colorectal-Histology}.
      The ratio of HPI by percentage (\textbf{HPI ratio}) computed by $v_d/\sum_{d^\prime=1}^{D} v_{d^\prime}$.
      The top-$3$ HPs are bolded.
      \textbf{Odd cols.} (Original):
      HPI by original f-ANOVA on $g(\xv) \coloneqq \min(f(\xv), f^{\gamma^\prime})$.
      \textbf{Even cols.} (Ours):
      HPI by PED-ANOVA.
    }
    \vspace{-4mm}
    \label{appx:experiments:tab:jahs-hpi-list}
    \makebox[0.98 \textwidth][c]{       
      \resizebox{0.98 \textwidth}{!}{
        \begin{tabular}{lc|cc|cc|c!{\vrule width 1pt}c|cc|cc|c}
                              & \multicolumn{12}{c}{HPI ratio (\%)}                                                                                                                                                                                                                                                                                           \\
          \toprule
          Dataset             & \multicolumn{6}{c}{\texttt{Fashion-MNIST}} & \multicolumn{6}{c}{\texttt{Colorectal-Histology}}                                                                                                                                                                                                                                \\
          \toprule
          Hyperparameter      & Normal                                     & \multicolumn{2}{c|}{Global 0.1}                   & \multicolumn{2}{c|}{Global 0.01} & Local          & Normal         & \multicolumn{2}{c|}{Global 0.1} & \multicolumn{2}{c|}{Global 0.01} & Local                                                                              \\

                              & Original                                   & Ours                                              & Original                         & Ours           & Original       & Ours                            & Original                         & Ours           & Original       & Ours           & Original       & Ours           \\
          \midrule
          Learning rate       & 5.66                                       & \textbf{12.39}                                    & 10.04                            & \textbf{11.81} & 15.12          & \textbf{13.30}                  & 1.64                             & \textbf{16.31} & \textbf{33.66} & \textbf{15.38} & \textbf{31.92} & \textbf{14.14} \\
          Weight decay        & 6.81                                       & 6.60                                              & 10.37                            & 4.76           & \textbf{15.55} & 3.15                            & 0.68                             & 2.96           & 5.07           & 2.42           & 3.72           & 1.58           \\
          Activation function & 0.76                                       & 0.34                                              & 0.14                             & 1.19           & 0.15           & 2.66                            & 0.01                             & 0.08           & 0.08           & 0.20           & 0.09           & 0.44           \\
          TrivialAugment      & 0.05                                       & 1.59                                              & 0.03                             & 9.87           & 0.06           & \textbf{26.73}                  & 0.11                             & \textbf{19.70} & 8.63           & \textbf{20.09} & \textbf{11.50} & \textbf{23.24} \\
          \midrule
          Depth multiplier    & 0.05                                       & 1.33                                              & 0.07                             & 0.47           & 0.16           & 0.01                            & 0.09                             & 2.05           & 1.88           & 4.09           & 1.65           & 7.47           \\
          Width multiplier    & 0.81                                       & \textbf{40.65}                                    & 3.84                             & \textbf{38.67} & 7.10           & \textbf{29.98}                  & 0.73                             & \textbf{45.33} & \textbf{36.04} & \textbf{39.72} & \textbf{40.46} & \textbf{29.58} \\
          \midrule
          Operation 1         & \textbf{15.85}                             & \textbf{11.54}                                    & \textbf{13.60}                   & \textbf{14.52} & 10.59          & 11.50                           & \textbf{11.33}                   & 3.25           & 2.24           & 3.96           & 1.66           & 4.54           \\
          Operation 2         & 1.62                                       & 3.93                                              & 1.56                             & 6.54           & 1.73           & 7.43                            & 4.85                             & 1.36           & 0.90           & 1.96           & 0.75           & 2.44           \\
          Operation 3         & \textbf{42.41}                             & 10.34                                             & \textbf{35.74}                   & 4.61           & \textbf{26.48} & 0.72                            & \textbf{64.26}                   & 5.16           & \textbf{8.85}  & 6.60           & 6.03           & 7.92           \\
          Operation 4         & 0.36                                       & 1.68                                              & 0.21                             & 2.26           & 0.31           & 3.07                            & 0.12                             & 0.47           & 0.25           & 0.39           & 0.30           & 0.30           \\
          Operation 5         & 3.11                                       & 0.91                                              & 1.97                             & 0.42           & 1.29           & 0.12                            & 4.72                             & 0.72           & 0.57           & 1.41           & 0.55           & 2.78           \\
          Operation 6         & \textbf{22.51}                             & 8.69                                              & \textbf{22.41}                   & 4.88           & \textbf{21.46} & 1.32                            & \textbf{11.47}                   & 2.61           & 1.82           & 3.78           & 1.36           & 5.57           \\
          \bottomrule
        \end{tabular}
      }
    }
  \end{center}
  \vspace{-3mm}
\end{table*}

\begin{figure*}
  \centering
  \includegraphics[width=0.98\textwidth]{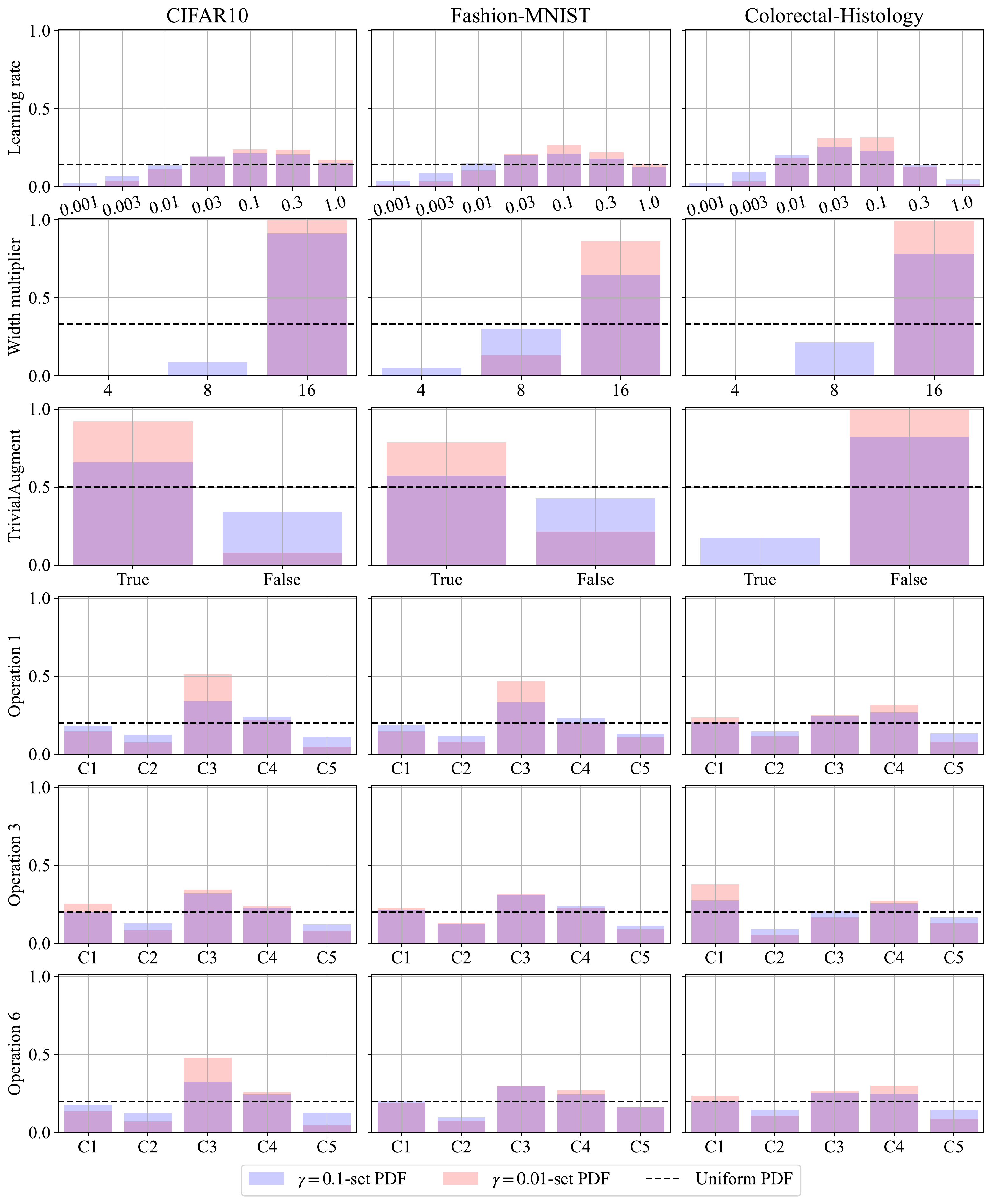}
  \caption{
    The distributions of important HPs in JAHS-Bench-201.
    The left, center, and right columns
    correspond to the results on
    \texttt{CIFAR10}, \texttt{Fashion-MNIST}, and \texttt{Colorectal-Histology}.
    The red shadows show the $\gamma=0.01$-set PDFs,
    the blue shadows show the $\gamma=0.1$-set PDFs, and
    the black dashed lines show the uniform PDFs.
    PED between a black line and a blue shadow
    is \texttt{Global 0.1},
    PED between a black line and a red shadow
    is \texttt{Global 0.01},
    and PED between a red shadow and a blue shadow
    is \texttt{Local}
    in Table~\ref{appx:experiments:tab:jahs-hpi-list}.
    Notice that C1 -- C5 correspond to the order of Table~\ref{appx:experiments:tab:jahs-search-space}
    and the overlap between the red and the blue shadows looks purple although they are separated shadows.
  }
  \label{appx:experiments:fig:jahs-histogram}
\end{figure*}

\subsection{Additional Results of Section~\ref{main:experiments:section}}
\label{appx:experiments:section:additional-results}
In this section, we show the additional results of the analysis on JAHS-Bench-201.
Table~\ref{appx:experiments:tab:jahs-hpi-list}
presents the HPI ratios of each HP on each dataset
and Figure~\ref{appx:experiments:fig:jahs-histogram}
visualizes the $\gamma$-set PDFs for each HP.
As in Section~\ref{main:experiments:section},
we check the RQs based on the results for each dataset
except \texttt{CIFAR10}.

For RQ1, we compare the column \texttt{(Global 0.1, Ours)} to \texttt{(Global 0.1, Original)} and
the column \texttt{(Global 0.01, Ours)} to \texttt{(Global 0.01, Original)} in Table~\ref{appx:experiments:tab:jahs-hpi-list}.
Although both PED-ANOVA and f-ANOVA provide similar sets of top HPs in \texttt{Colorectal-Histology},
we yielded different patterns in \texttt{Fashion-MNIST}.
Since the $\gamma$-set PDFs for \wm\ are peaked at $16$ in Figure~\ref{appx:experiments:fig:jahs-histogram},
\wm\ should be large in f-ANOVA;
however, f-ANOVA did not provide such interpretation.
This was probably due to the fact that the performance on \texttt{Fashion-MNIST} is already saturated
at the top-$10\%$ as seen in Figure~\ref{appx:experiments:fig:datasets-performance-dist},
and thus our method can provide similar sets of important HPs unless the performance is already saturated at a given quantile $\gamma^\prime$.

For RQ2, we compare the column \texttt{(Global 0.1, Ours)} and \texttt{(Global 0.01, Ours)} to \texttt{(Normal, Original)}.
As discussed in Section~\ref{main:experiments:section},
we found the misclassification of \texttt{Op.3} as the most important and
of \ta\ as the least important HP on \texttt{CIFAR10} and the same phenomenon happens to both datasets.
Therefore, scale invariance indeed helps to successfully identify HPI for HPs that would have been misclassified by the \texttt{(Normal, Original)} setting.

For RQ3, we compare the column \texttt{(Global 0.01, Ours)} to \texttt{(Local, Ours)}.
We observe that the HPI of \wm\ drops from the \texttt{Global 0.01} setting to the \texttt{Local} setting.
Simultaneously, the HPI of \ta\ increases sharply across the same.
This suggests that optimizing \wm\ is no longer important when moving from the top-$10\%$ to the top-$1\%$ performance
but optimizing \ta\ is very important.
The reason behind this change becomes clear when we observe the change in $\gamma$-set PDFs of the two HPs in Figure~\ref{appx:experiments:fig:jahs-histogram}.
Both the $\gamma$-set PDFs for \wm\ are sharply peaked at $16$,
indicating that no further optimization is needed on \wm.
However, the $\gamma$-set PDFs for \ta\ only start peaking at the value \texttt{True} for the $\gamma=0.01$-set PDF especially in \texttt{Fashion-MNIST}.
This clearly demonstrates that local HPI is necessary for deriving the correct interpretation in the top-$\gamma^\prime$ quantiles,
since \texttt{(Local, Ours)} successfully identifies the relative importance of optimizing the two HPs.
Last but not least, if both global and local HPI with wished quantiles $\gamma, \gamma^\prime$ exhibits low values,
removing such HPs, e.g. \texttt{Activation function}, is expected to have a less negative impact although it is insecure to remove HPs, e.g. \ta\ in \texttt{Fashion-MNIST}, only by looking at global HPI.

%% file: appendices/use-cases.tex
\section{Practical Usecases and Limitations}
\label{appx:use-cases:section}
In this section, we first discuss the advantages
and the limitations of our method
and then describe the usecases of our method.

\subsection{Advantages and Limitations}
We list the advantages
and limitations, which are not discussed in the main paper.
The advantages of our method are that:
\begin{enumerate}
  \item our method can handle multi-output functions and
  classification problems because
  we can measure HPI as long as the objective function
  can be divided by a specific threshold,
  \item the meaning of HPI, which is how important each HP is to achieve the top-$\gamma^\prime$
  quantile, is more clear compared to global f-ANOVA~\citeappx{hutter2014efficient},
  \item we can handle infinity and missing values without
  preprocessing thanks to the scale ignorance nature, and
  \item the implementation is very simple.
\end{enumerate}
On the other hand, the limitations of our method are that:
\begin{enumerate}
  \item since the computation of higher orders of HPI requires
  exponential amounts of time complexity $O(\prod_{d \in s} n_d^2)$~\footnote{
    We can reduce the complexity with some efforts or Monte-Carlo sampling.
  } if we analyze the interaction in $\X_s$,
  our method is not suitable for the approximation of higher order HPI with many grids $n_d$,
  \item the exact percentage, i.e. $v_s/v_0$, is not available
  in high dimensions because it is hard to obtain 
  the precise variance $v_0$, i.e. the highest order HPI, and
  \item the marginal objective function (see Figure 1 in \citewithname{hutter2014efficient})
  is not available (instead, we yield distributional visualizations as in Figure~\ref{main:validation:fig:local-validation}).
\end{enumerate}
Note that even when we do not have $v_0$,
the ratio between two marginal variances
$v_d / v_{d^\prime}$
have the same meaning as the ratio of percentages.
$(v_d / v_0) / (v_{d^\prime} / v_0) = v_d / v_{d^\prime}$.

\subsection{Post-Hoc Analysis of HPO}
\label{appx:use-cases:section:posthoc-hpo}
While post-hoc analysis of HPO is an obvious application for our method,
we would like to note that
the appropriate usage of our method
is to apply local HPI with $\gamma = 1$ rather than to
apply global HPI in Eq.~(\ref{appx:theoretical-details:eq:global-hpi}).
More specifically, we should use $p_d(x_d | \D)$, which is a KDE built by the whole observations $\D$,
instead of $u(\X_d)$.
This relates to the sampling bias of (non-random) HPO methods~\citeappx{moosbauer2021explaining}.
For example, since Bayesian optimization
tries to exploit knowledge, the samples would not be generated from the uniform distribution
and they concentrate in local spaces.
Therefore, $p_d(x_d | \D)$ is often dissimilar to the uniform PDF $u(\X_d)$
and global HPI in Eq.~(\ref{appx:theoretical-details:eq:global-hpi}) is biased.
When applying PED-ANOVA to samples obtained by Bayesian optimization,
high HPI in global HPI indicates that the corresponding HP was easy to optimize even with random search,
and high HPI in local HPI with $\gamma = 1$ indicates that the corresponding HP was particularly paid attention to by the sampler.
If the goal of the analysis is to identify easy-to-find important HPs,
global HPI is appropriate,
but if the goal of the analysis is to identify HPs that should be intensively searched,
local HPI with $\gamma = 1$ is more appropriate.

In the same vein, we can apply our method to the results of constrained optimization.
Typically, we would like to know the HPI in feasible domains.
In such a case,
we first build the PDF for feasible domains $p_d(x_d | \X_{d,\mathrm{feasible}})$
as local space
and we define another PDF $p_d(x_d | \X_{d,\mathrm{feasible}}^{\gamma})$, e.g.
the PDF of the top-$\gamma$-quantile configurations in feasible domains.
Then we can compute HPI via PED.

\subsection{Search Space Reduction}
\label{appx:use-cases:section:space-reduction}
Our method is useful for search space designs
and we would like to discuss the possibility
using Figure~\ref{main:validation:fig:local-validation}
in Section~\ref{main:validation:section:performance-validation}.
In the figures,
we presented the uniform PDF,
$\gamma$-set PDFs, and
$\gamma^\prime$-set PDFs
built using the random samples
where $\gamma = 0.1, \gamma^\prime = 0.01$.
We could approach search space reduction from two perspectives:
(1) domain reduction of each HP using global HPI
and (2) HP selection using local HPI.

For the domain reduction,
we need to pay attention to HPs with higher global HPI.
In this example, $x_1$ and $x_2$ have higher global HPI,
and both $\gamma$-set PDFs have a peak.
Practitioners could sample these HPs only from
the domains where the $\gamma$-set PDFs exhibit a larger value than the uniform PDF, i.e. $p_d(x_d | \D^\gamma) \geq u_d$.

On the other hand for the HP selection,
we need to pay attention to HPs with low HPI both in the local and global spaces.
For example, while we can conclude that $x_3, x_4$ are not important
if we rely only on global HPI,
$x_3$ is more important than $x_1$ in the local space as seen in the figure.
For this reason, we should not discard $x_3$;
however, as the local HPI of $x_4$ is small and close to zero,
it makes sense to remove $x_4$ from the search space
and fixes $x_4$ to a default value.
Note that it is also not appropriate to discard HPs with low HPI only in the local space
because low local HPI just implies that the HPs have a similar trend
both in the global and local spaces and the HPs are likely to have more interaction effects compared to the other HPs.

\subsection{Exploratory Data Analysis}
When we have a dataset $\{(\xv_n, y_n)\}_{n=1}^N$,
we can first define a KPI $y^\star \coloneqq y^\gamma$
and pick the data points such that $y_n \leq y^\star$
and define it as $\D^\gamma$.
Then we obtain the marginal $\gamma$-set PDFs
$p_d(x_d|\D^\gamma)$ for each feature in the global scale.
By plotting those PDFs independently,
we can know what features 
might lead to a good KPI in which ranges
and we can also know what features might not
affect the KPI.

For the usage of local HPI,
first recall that $\Xg, \Xgp$ only require $\Xgp \subseteq \Xg$,
and thus we can define local spaces such that
$\D^\gamma = \{\xv_n | \xv_n \in \D, y^{\gamma_1} \leq y_n \leq y^{\gamma_2}\}$
and $\D^{\gamma^\prime} = \{\xv_n | \xv_n \in \D, y^{\gamma^\prime_1} \leq y  \leq y^{\gamma^\prime_2}\}$
where $y^{\gamma_1} \leq y^{\gamma^\prime_1} \leq y^{\gamma^\prime_2} \leq y^{\gamma_2}$.
This analysis allows practitioners to know what HPs might be
the key to improving the KPI in the local space.